\documentclass{article}
\usepackage{psfrag}
\usepackage[parfill]{parskip}
\usepackage{float, graphicx}
\usepackage[]{epsfig}
\usepackage{amsmath, amsthm, amssymb,}
\usepackage{epsfig}
\usepackage{verbatim}
\usepackage{multicol}
\usepackage{url}
\usepackage{latexsym}
\usepackage{mathrsfs}
\usepackage[colorlinks, bookmarks=true]{hyperref}
\usepackage{graphicx}
\usepackage{amsmath}
\usepackage{enumerate}
\usepackage[normalem]{ulem}
\usepackage{bm}
\usepackage{dsfont}
\usepackage{stmaryrd}
\usepackage{cite}

\usepackage{color}

\begingroup
\makeatletter
   \@for\theoremstyle:=definition,remark,plain\do{%
     \expandafter\g@addto@macro\csname th@\theoremstyle\endcsname{%
        \addtolength\thm@preskip\parskip
     }%
   }
\endgroup

\numberwithin{equation}{section}

\renewcommand{\cal}{\mathcal}

\newcommand{\cN}{{\cal N}}

\newcommand\cX{{\mathcal X}}

\newcommand{\fc}{{\mathtt c}}
\newcommand{\fC}{{\mathtt C}}

\newcommand{\fa}{{\frak a}}

\newcommand{\fD}{{\frak D}}
\newcommand{\fK}{{\frak K}}

\newcommand{\sfe}{{\mathsf e}}
\newcommand{\sfd}{{\mathsf d}}
\newcommand{\sff}{{\mathsf f}}
\newcommand{\sfv}{{\mathsf v}}
\newcommand{\sfu}{{\mathsf u}}

\newcommand{\rd}{{\rm d}}

\newcommand{\bE}{\mathbb{E}}

\newcommand{\bR}{{\mathbb R}}

\newcommand{\al}{\alpha}

\newcommand{\la}{\lambda}

\DeclareMathOperator{\diag}{diag}

\DeclareMathOperator{\OO}{O}

\newcommand{\deq}{\mathrel{\mathop:}=} 
 
\renewcommand{\leq}{\leqslant}
\renewcommand{\geq}{\geqslant}

\definecolor{darkred}{rgb}{0.9,0,0.3}
\definecolor{darkblue}{rgb}{0,0.3,0.9}
\definecolor{purple}{rgb}{0.7,0,0.6}
\definecolor{darkyellow}{rgb}{0.8,0.8,0}

\newcommand{\nc}{\normalcolor}

\newcommand{\del}{\partial}

\newcommand{\beq}{\begin{equation}}
\newcommand{\bEq}{\end{equation}}

\theoremstyle{plain} 
\newtheorem{theorem}{Theorem}[section]
\newtheorem*{theorem*}{Theorem}
\newtheorem{lemma}[theorem]{Lemma}
\newtheorem*{lemma*}{Lemma}
\newtheorem{corollary}[theorem]{Corollary}
\newtheorem*{corollary*}{Corollary}
\newtheorem{proposition}[theorem]{Proposition}
\newtheorem*{proposition*}{Proposition}
\newtheorem{assumption}[theorem]{Assumption}
\newtheorem*{assumption*}{Assumption}
\newtheorem{claim}[theorem]{Claim}

\newtheorem*{definition*}{Definition}

\newtheorem*{example*}{Example}
\newtheorem{remark}[theorem]{Remark}

\newtheorem*{remark*}{Remark}
\newtheorem*{remarks*}{Remarks}

\def\author#1{\par
    {\centering{\authorfont#1}\par\vspace*{0.05in}}
}

\def\titlefont{\fontsize{13}{15}\bfseries\boldmath\selectfont\centering{}}
\def\authorfont{\fontsize{13}{15}}

\let\affiliationfont\rhfont

\def\address#1{\par
    {\centering{\affiliationfont#1\par}}\par\vspace*{11pt}
}

\def\body{
\setcounter{footnote}{0}
\def\thefootnote{\alph{footnote}}
\def\@makefnmark{{$^{\rm \@thefnmark}$}}
}

\def\title#1{
    \thispagestyle{plain}
    \vspace*{-14pt}
    \vskip 79pt
    {\centering{\titlefont #1\par}}%
    \vskip 1em
}

\newcommand{\val}{{\rm val}}
\newcommand{\LinMul}{{\rm LinMul}}
\newcommand{\Mul}{{\rm Mul}}
\newcommand{\LinProd}{{\rm LinProd}}
\newcommand{\poly}{{\rm poly}}

\begin{document}

\title{ Dynamics of Deep Neural Networks and  Neural Tangent Hierarchy}

\vspace{1.2cm}

\noindent \begin{minipage}[c]{0.5\textwidth}
 \author{Jiaoyang Huang}
\address{IAS\\
   E-mail: jiaoyang@ias.edu}
 \end{minipage}
\begin{minipage}[c]{0.5\textwidth}
 \author{Horng-Tzer Yau}
\address{Harvard University \\
   E-mail: htyau@math.harvard.edu}

 \end{minipage}

\begin{abstract}
The evolution of a deep neural network trained by the gradient descent can be described by its neural tangent kernel (NTK) as introduced in \cite{jacot2018neural}, where it was proven that in the infinite width limit the NTK converges to an explicit limiting kernel and it stays constant during training. The NTK was also implicit in some other recent
papers \cite{du2018gradient1,du2018gradient2,arora2019fine}. In the overparametrization regime, a fully-trained deep neural network is indeed equivalent to the kernel regression predictor using the limiting NTK. And the gradient descent achieves zero training loss for a deep overparameterized neural network. However, it was observed in \cite{arora2019exact} that there is a performance gap between the kernel regression using the limiting NTK and the deep neural networks. This performance gap  is likely to  originate   from the change of the NTK along training due to the finite width effect. The change of the NTK along the  training is central to describe the generalization features of deep neural networks. 

In the current paper, we study the dynamic of the NTK for finite width deep fully-connected neural networks. 
 We derive an  infinite hierarchy of ordinary differential equations, the neural tangent hierarchy (NTH)
which captures the   gradient descent  dynamic of the deep neural network. 
Moreover, under certain conditions 
on the neural network width and the data set dimension,  we prove  that the truncated hierarchy of NTH approximates the dynamic of the NTK up to arbitrary precision. This description makes it possible to directly study the change of the NTK for deep neural networks, and sheds light on the observation that deep neural networks outperform kernel regressions using the corresponding limiting NTK. 
\end{abstract}

\let\thefootnote\relax\footnote{\noindent The work of H.-T. Y. is partially supported by NSF Grants DMS-1606305 and DMS-1855509, and a Simons Investigator award.}

\section{Introduction}

Deep neural networks have become popular due to their unprecedented success in a variety of machine learning tasks. Image recognition \cite{lecun1998gradient, krizhevsky2012imagenet, szegedy2015going},  
speech recognition \cite{hinton2012deep,sainath2013deep}, playing Go \cite{silver2016mastering,silver2017mastering} and natural language understanding \cite{collobert2011natural, wu2016google, devlin2018bert} are just a few of the recent achievements. 
However, one aspect of deep neural networks that  is  not  well understood is training. Training a deep neural network is usually done via a gradient decent  based algorithm. Analyzing such training dynamics is challenging. Firstly, as highly nonlinear structures, deep neural networks usually involve a large number of parameters. Secondly, as highly non-convex optimization problems, there is no guarantee that a gradient based algorithm will be able to find the optimal parameters efficiently during the training of neural networks. \emph{One question then arises: given such complexities, is it possible to obtain a succinct description of the training dynamics?}

In this paper, we focus on the empirical risk minimization problem with the quadratic loss function
\begin{align*}
\min_\theta L(\theta)=\frac{1}{2n}\sum_{\al=1}^n (f(x_\al,\theta)-y_\al)^2,
\end{align*}
where $\{x_\al\}_{\al=1}^n$ are the training inputs, $\{y_\al\}_{\al=1}^n$ are the labels, and the dependence is modeled by a deep fully-connected feedforward neural network with $H$ hidden layers. 
The network has $d$ input nodes, and the input vector is given by $x\in {\mathbb R}^{d}$. For $1\leq \ell\leq H$, the  $\ell$-th hidden layer has $m$ neurons. Let $x^{(\ell)}$ be the output of the $\ell$-th layer with $x^{(0)}=x$. Then the feedforward neural network is given by the set of recursive equations:
\begin{equation}\label{e:defDNN}
x^{(\ell)}=\frac{1}{\sqrt{m}}\sigma(W^{(\ell)}x^{(\ell-1)}),\quad \ell=1,2,\cdots, H,
\end{equation}
where $W^{(\ell)}\in \bR^{m \times d}$ if $\ell=1$ and $W^{(\ell)}\in \bR^{m\times m}$ if $2\leq \ell\leq H$ are the weight matrices, and $\sigma$ is the activation unit, which is applied coordinate-wise to its input.  The output of the neural network is 
\begin{equation}\label{eq:f}
f(x,\theta)=a^\top x^{(H)} \in {\mathbb R}, 
\end{equation}
where $a\in \bR^{m}$ is the weight matrix for the output layer. 
We denote the vector containing all trainable parameters by $\theta=({\rm vec}(W^{(1)}), {\rm vec}(W^{(2)})\dots, {\rm vec}(W^{(H)}),a)$. We remark that this parametrization is nonstandard because of those $1/\sqrt m$ factors. However, it has already been adopted in several recent works \cite{jacot2018neural, du2018gradient1,du2018gradient2,lee2019wide}. We note that the predictions and training dynamics of \eqref{e:defDNN} are identical to those of standard networks, up to a scaling factor $1/\sqrt m$ in the learning rate for each parameter.

We initialize the neural network with random Gaussian weights following the Xavier initialization scheme \cite{glorot2010understanding}. More precisely, we set the initial parameter vector $\theta_0$  as $W^{(\ell)}_{ij} \sim \mathcal N(0, \sigma^2_w)$, $a_i\sim \mathcal N(0,\sigma_a^2)$. In this way, for the randomly initialized neural network, we have that the $L_2$ norms of the output of each layer are of order one, i.e. $\|x^{(\ell)}\|_2^2=\OO(1)$ 
for $0\leq \ell\leq H$, and $f(x,\theta_0)=\OO(1)$ with high probability. In this paper, we train all layers of the neural network with continuous time gradient descent (gradient flow): for  any  time $t\geq 0$ 
\begin{align}\label{e:gd}
\del_t W^{(\ell)}_t=- \del_{W^{(\ell)}}L(\theta_t),\quad \ell=1,2,\cdots, H, \quad \del_t a_t=- \del_{a}L(\theta_t),
\end{align}
where 
$\theta_t=({\rm vec}(W_t^{(1)}), {\rm vec}(W^{(2)}_t)\dots, {\rm vec}(W^{(H)}_t),a_t)$.

For simplicity of notations, we write $\sigma(W^{(\ell)}x^{(\ell-1)})$ as $\sigma_\ell(x)$, or simply $\sigma_\ell$ if the context is clear. We write its derivative $\diag(\sigma'(W^{(\ell)}x^{(\ell-1)}))$ as $\sigma'_\ell(x)=\sigma_\ell^{(1)}(x)$, and $r$-th derivative $\diag(\sigma^{(r)}(W^{(\ell)}x^{(\ell-1)}))$ as $\sigma^{(r)}_\ell(x)$, or $\sigma^{(r)}_\ell$ for $r\geq 1$. In this notation, $\sigma^{(r)}_\ell(x)$ are diagonal matrices. With those notations, explicitly, the continuous time gradient descent dynamic \eqref{e:gd} is 
\begin{align}\begin{split}\label{e:derW}
&\phantom{{}={}}\del_t W^{(\ell)}_t
=- \del_{W^{(\ell)}}L(\theta_t)\\
&=- \frac{1}{n}\sum_{\beta=1}^n
\left(\sigma'_\ell(x_\beta)\frac{(W_t^{(\ell+1)})^\top}{\sqrt m}\cdots \sigma'_{H}(x_\beta)\frac{a_t}{\sqrt m}\right)
\otimes (x_\beta^{(\ell-1)})^{\top}(f(x_\beta,\theta_t)-y_\beta),
\end{split}\end{align}
for $\ell=1,2,\cdots, H$, and 
\begin{align}\label{e:dera}
\del_t a_t=- \del_{a}L(\theta_t)=- \frac{1}{n}\sum_{\beta =1}^n x_\beta^{(H)}(f(x_\beta,\theta_t)-y_\beta).
\end{align}

\subsection{Neural Tangent Kernel}

A recent paper \cite{jacot2018neural} introduced
the Neural Tangent Kernel (NTK) and proved the limiting NTK captures the behavior of fully-connected deep neural networks in the infinite width limit trained by gradient descent: 
\begin{align}\begin{split}\label{e:descent}
&\phantom{{}={}}\del_t f(x,\theta_t)
=\del_\theta f(x,\theta_t) \del_t \theta_t
=-\del_\theta f(x,\theta_t) \del_\theta L(\theta_t)\\
&=-\frac{1}{n}\del_\theta f(x,\theta_t) \sum_{ \beta =1}^n \del_\theta f(x_\beta, \theta_t) (f(x_\beta, \theta_t)-y_\beta)
=-\frac{1}{n}\sum_{\beta=1}^n K^{(2)}_t(x,x_\beta) (f(x_\beta, \theta_t)-y_\beta),
\end{split}\end{align}
where the NTK $K_t^{(2)}(\cdot, \cdot)$ is given by 
\begin{align}\begin{split}\label{e:defK2}
 &\phantom{{}={}}K^{(2)}_t(x_\al,x_\beta)  
 =\langle \del_\theta f(x_\al,\theta_t) , \del_\theta f(x_\beta, \theta_t)\rangle=\sum_{\ell=1}^{H+1}G_t^{(\ell)}(x_\al, x_\beta)
\end{split}\end{align}
and for $1\leq \ell\leq H$,
\begin{align*}\begin{split}
&\phantom{{}={}}G_t^{(\ell)}(x_\al, x_\beta)= \langle \del_{W^{(\ell)} } f(x_\al,\theta_t) ,  \del_{W^{(\ell)} }  f(x_\beta, \theta_t)\rangle\\
&= \left\langle
\sigma'_\ell(x_\al)\frac{(W_t^{(\ell+1)})^\top}{\sqrt m}\cdots \sigma'_{H}(x_\al)\frac{a_t}{\sqrt m}, \sigma'_\ell(x_\beta)\frac{(W_t^{(\ell+1)})^\top}{\sqrt m}\cdots \sigma'_{H}(x_\beta)\frac{a_t}{\sqrt m}  \right\rangle \langle x^{(\ell-1)}_\al, x^{(\ell-1)}_\beta \rangle
\end{split}\end{align*}
and 
\begin{align*}
G_t^{(H+1)}=\langle \del_a f(x_\al, \theta_t), \del_a f(x_\beta,\theta_t)\rangle=\langle x_\al^{(H)}, x_\beta^{(H)}\rangle.
\end{align*}
The NTK $K_t^{(2)}(\cdot, \cdot)$ varies along training. However, in the infinite width limit, the training dynamic is very simple: The NTK does not change along training, $K_t^{(2)}(\cdot, \cdot)=K_\infty^{(2)}(\cdot, \cdot)$. The network function $f(x,\theta_t)$ follows a linear differential equation \cite{jacot2018neural}:
\begin{align}\label{e:inflimit}
\del_t f(x,\theta_t)
=-\frac{1}{n}\sum_{\beta=1}^n K^{(2)}_\infty(x,x_\beta) (f(x_\beta, \theta_t)-y_\beta),
\end{align}
which becomes analytically tractable. In other words, the training dynamic is equivalent to the kernel regression using the limiting NTK $K_\infty^{(2)}(\cdot, \cdot)$. 
While the linearization \eqref{e:inflimit} is only exact in the infinite width limit, for a sufficiently wide deep neural network, \eqref{e:inflimit} still provides a good approximation of the learning dynamic for the corresponding deep neural network \cite{du2018gradient1,du2018gradient2, lee2019wide}. As a consequence, it was proven in \cite{du2018gradient1,du2018gradient2} that,  for a fully-connected wide neural network with $m\gtrsim n^4$ under certain assumptions on the data set,  the gradient descent converges to zero training loss at a linear rate. Although highly overparametrized neural networks is equivalent to the kernel regression, it is possible to show that the class of finite width neural networks is more expressive than the limiting NTK. It has been constructed in \cite{ghorbani2019linearized,yehudai2019power,allen2019can} that there are simple functions that can be efficiently learnt by finite width neural networks, but not the kernel regression using the limiting NTK.

\subsection{Contribution}
There is a performance gap between the kernel regression \eqref{e:inflimit} using the limiting NTK and the deep neural networks. It was observed in \cite{arora2019exact} that  the convolutional neural networks outperform their corresponding limiting NTK by $5\%$ - $6\%$. This performance gap is likely to originate  from the change of the NTK along training due to the finite width effect. The change of the NTK along training has its benefits on generalization. 

In the current paper, we study the dynamic of the NTK for finite width deep fully-connected neural networks.
Here we summarize our main contributions:
\begin{itemize}
\item We show the 
  gradient descent  dynamic 
is captured by an infinite hierarchy of ordinary differential equations, the neural tangent hierarchy (NTH).
Different from the limiting NTK \eqref{e:defK2}, which depends only on the neural network architecture, the NTH is data dependent and capable of learning data-dependent features.

\item
We derive  a priori 
 estimates of the higher order kernels involved in the NTH.  Using these a priori  estimates as input, we confirm a numerical observation in \cite{lee2019wide} that the NTK varies at a rate of order $\OO(1/m)$. As a corollary, this implies that for a fully-connected wide neural network with $m\gtrsim n^3$, the gradient descent converges to zero training loss at a linear rate,   which improves the results in \cite{du2018gradient2}.

\item
The NTH is just an infinite sequence of  relationship. Without truncation, it cannot be used to determine the dynamic of the NTK.
 Using the a priori estimates of the higher order kernels as input, we construct a  truncated hierarchy of ordinary differential equations, the truncated NTH. We show that  this system of truncated equations approximates the dynamic of the NTK to certain time up to arbitrary precision.  
This description makes it possible to directly study the change of the NTK for deep neural networks.

\end{itemize}

\subsection{Notations}

In the paper, we fix a large constant $p>0$, which appears in Assumptions \eqref{a:sigmaasup} and \eqref{a:nonlinear}. 
We use $\fc, \fC$ to represent universal constants, which might be different from line to line. In the paper, we write $a=\OO(b)$ or $a\lesssim b$ if there exists some large universal constant $\fC$ such that $|a|\leq \fC b$. We write $a\gtrsim b$ if there exists some small universal constant $\fc>0$ such that $a\geq \fc b$. We write $a\asymp b$ if there exist universal constants $\fc,\fC$ such that $\fc b\leq |a|\leq \fC b$. We reserve $n$ for the number of input samples and $m$ for the width of the neural network. For practical neural networks, we always have that $m\lesssim \poly(n)$ and $n\lesssim \poly(m)$. We denote the set of input samples as $\cX=\{x_1, x_2, \cdots, x_n\}$. For simplicity of notations, we write the output of the neural network as $f_\beta(t)=f(x_\beta,\theta_t)$.  We denote vector $L_2$ norm as $\|\cdot\|_2$, vector or function $L_\infty$ norm as $\|\cdot\|_\infty$, matrix spectral norm as $\|\cdot\|_{2\rightarrow 2}$, and matrix Frobenius norm as $\|\cdot\|_{\rm F}$. We say that an event holds with high probability, if it holds with probability at least $1-e^{-m^{\fc}}$ for some $\fc>0$. Then the intersection of $\poly(n,m)$ many high probability events is still a high probability event, provided $m$ is large enough. In the paper, we treat 
$\fc_r, \fC_r$ in Assumption \ref{a:sigmaasup} and \ref{a:nonlinear}, and the depth $H$ as constants. We will not keep track of them.

\subsection{Related Work}
In this section, we survey an incomplete list of previous works on optimization aspect of deep neural networks.

Because of the highly non-convexity nature of deep neural networks, the gradient based algorithms can potentially get stuck near a critical point, i.e., saddle point or local minimum. So one important question in deep neural networks is: what does the loss landscape look like. 
One promising candidate for loss landscapes is the class of functions that satisfy: (i) all local minima are global minima and (ii) there exists a negative curvature for every saddle point. 
A line of recent results show that, in many optimization problems of interest \cite{ge2015escaping,ge2016matrix,sun2018geometric,sun2016complete,bhojanapalli2016global,park2016non}, loss landscapes are in such class. For this function class, (perturbed) gradient descent \cite{jin2017escape,ge2015escaping,lee2016gradient} can find a global minimum. However, even for a three-layer linear network, there exists a saddle point that does not have a negative curvature \cite{kawaguchi2016deep}. So it is unclear whether this geometry-based approach can be used to obtain the global convergence guarantee of first-order methods. Another approach is to show that  practical deep neural networks allow some additional structure or assumption to make non-convex optimizations tractable. Under certain simplification assumptions,  it has been proven  recently that there are novel loss landscape structures in deep neural networks, which may play a role in making the optimization tractable \cite{dauphin2014identifying,choromanska2015loss, kawaguchi2016deep,liang2018adding, kawaguchi2019elimination}.

Recently, it was proved in a series of   papers that, if
the size of a neural network is significantly larger than the
size of the dataset, the (stochastic) gradient descent algorithm
can find optimal parameters \cite{li2018learning,du2018gradient1,song2019quadratic,du2018gradient2,allen2018convergence,zou2018stochastic}. In the overparametrization regime, a fully-trained deep neural network is indeed equivalent to the kernel regression predictor using the limiting NTK \eqref{e:inflimit}. As a consequence, the gradient descent achieves zero training loss for a deep overparameterized neural network. Under further assumptions, it can be shown that the trained networks generalize \cite{arora2019fine,allen2018learning}. Unfortunately, there is a significant gap between the overparametrized neural networks, which are provably trainable, and neural networks in common practice. Typically, deep neural networks used in practical applications are trainable, and yet, much
smaller than what the previous theories require to ensure
trainability. In \cite{kawaguchi2019gradient}, it is proven that gradient
descent can find a global minimum for certain deep neural
networks of sizes commonly encountered in practice.

Training dynamics of neural networks in the mean field setting have been studied in \cite{mei2019mean,song2018mean, araujo2019mean,nguyen2019mean,sirignano2019mean,chizat2018global}.
 Their mean field analysis describes distributional dynamics of neural network parameters via certain nonlinear partial differential equations, in the asymptotic regime of large network sizes and large number of stochastic gradient descent training iterations. However,
their analysis is restricted to neural networks in the mean-field framework with a normalization factor $1/m$, different from ours $1/\sqrt m$, which is commonly used in modern networks \cite{glorot2010understanding}.

\section{Main results}

\begin{assumption}\label{a:sigmaasup}
The activation function $\sigma$ is smooth, and for any $1\leq r\leq 2p+1$, there exists a constant $\fC_r>0$ such that the $r$-th derivative of $\sigma$ satisfies
$
\|\sigma^{(r)}(x)\|_\infty\leq \fC_r.
$
\end{assumption}

Assumption \ref{a:sigmaasup} is  satisfied by using common  activation units such as sigmoid and hyperbolic tangents. Moreover, the softplus activation,  which is defined as $\sigma_a(x)=\ln(1+\exp(a x))/a$,  satisfies Assumption \ref{a:sigmaasup} with any hyperparameter $a \in \bR_{>0}$. The softplus activation can approximate the ReLU activation  for any desired accuracy as 
\begin{align*}
\sigma_{a}(x) \rightarrow \mathrm{relu}(x) \text{ as } a\rightarrow \infty,
\end{align*}
where $\mathrm{relu}$ represents the ReLU activation.

\begin{assumption}\label{a:nonlinear}
There exists a small constant $\fc>0$ such that the training inputs satisfy $\fc<\|x_\al\|_2\leq \fc^{-1}$. For any $1\leq r\leq 2p+1$, there exists a constant $\fc_r>0$ such that for any distinct indices $1\leq \al_1,\al_2,\cdots, \al_r\leq n$, the smallest singular value of the data matrix $[x_{\al_1}, x_{\al_2},\cdots, x_{\al_r}]$ is at least $\fc_r$.
\end{assumption}

For more general input data, we can always normalize them such that 
$\fc<\|x_\al\|_2\leq \fc^{-1}$. Under this normalization, for the randomly initialized deep neural network, it holds that $\|x_\al^{(\ell)}\|_2=\OO(1)$ for all $1\leq \ell\leq H$, where the implicit constants depend on $\ell$. The second part of Assumption \ref{a:nonlinear} requires that for any small number of input data: $x_{\al_1}, x_{\al_2},\cdots, x_{\al_r}$, they are linearly independent.

\begin{theorem}\label{t:main1}
Under Assumptions \ref{a:sigmaasup} and \ref{a:nonlinear}, there exists an infinite family of operators $K_t^{(r)}: \cX^r\mapsto \bR$ for $r\geq 2$, the continuous time gradient descent dynamic is given by an infinite hierarchy of ordinary differential equations, i.e., the NTH, 
\begin{align}\label{e:dynamic}
&\del_t(f_\al(t)-y_{\al})=-\frac{1}{n} \sum_{\beta=1}^n K_t^{(2)}(x_\al, x_\beta)(f_\beta(t)-y_\beta),
\end{align} 
and for any $r\geq 2$,
\begin{align}\begin{split}\label{e:dynamicr}
&\del_tK_t^{(r)}(x_{\al_1},x_{\al_2},\cdots, x_{\al_r})=-\frac{1}{n} \sum_{\beta=1}^n K^{(r+1)}_t(x_{\al_1}, x_{\al_2}, \cdots, x_{\alpha_r}, x_\beta)(f_\beta(t)-y_\beta).
\end{split}\end{align}
There exists a deterministic family (independent of $m$) of operators $\fK^{(r)}: \cX^r\mapsto \bR$ for $2\leq r\leq p+1$  and $\fK^{(r)}=0$ if $r$ is odd, such that with high probability with respect to the random initialization, there exist some constants $\fC, \fC'>0$ such that 
\begin{align}\label{e:tprior1}
\left\|K^{(r)}_0-\frac{\fK^{(r)}}{m^{r/2-1}}\right\|_{\infty}\lesssim \frac{(\ln m)^\fC}{m^{(r-1)/2}},
\end{align}
and for $0\leq t\leq m^{\frac{p}{2(p+1)}}/(\ln m)^{\fC'}$,  
\begin{align}\label{e:tprior2}
\|K^{(r)}_t\|_{\infty}\lesssim \frac{(\ln m)^\fC}{m^{r/2-1}}.
\end{align}
\end{theorem}

It was proven in \cite{du2018gradient2, lee2019wide} that  the change of the NTK for a wide deep neural network is upper bounded by $\OO(1/\sqrt m)$. However, the numerical experiments in \cite{lee2019wide} indicate the change of the NTK is closer to $\OO(1/m)$. As a corollary of Theorem \ref{t:main1}, we confirm the numerical observation that the NTK varies at a rate of order $\OO(1/m)$. 
\begin{corollary}\label{c:change}
Under Assumptions \ref{a:sigmaasup} and \ref{a:nonlinear}, the NTK $K_t^{(2)}(\cdot, \cdot)$ varies at a rate of order $\OO(1/m)$: with high probability with respect to the random initialization, there exist some constants $\fC, \fC'>0$ such that  for $0\leq t\leq m^{\frac{p}{2(p+1)}}/(\ln m)^{\fC'}$, it holds
\begin{align*}
\|\del_t K_t^{(2)}\|_\infty\lesssim \frac{(1+t)(\ln m)^\fC}{m}.
\end{align*}

\end{corollary}

As another corollary of Theorem \ref{t:main1}, for a fully-connected wide neural network with $m\gtrsim n^3$, the gradient descent converges to zero training loss at a linear rate. 
\begin{corollary}\label{c:zeroloss}
Under Assumptions \ref{a:sigmaasup} and \ref{a:nonlinear}, we further assume that there exists $\la>0$ (which might depend on $n$)  
\begin{align}\label{e:eigasup}
\la_{\min} \left[K_0^{(2)}(x_\al, x_\beta)\right]_{1\leq \al, \beta\leq n}\geq \la,
\end{align}
and the width $m$ of the neural network satisfies
\begin{align}\label{mn}
m\geq \fC'\left(\frac{n}{\la}\right)^3(\ln m)^{\fC}\ln(n/\varepsilon)^2,
\end{align} 
for some large constants $\fC, \fC'>0$. Then with high probability with respect to the random initialization, the training error decays exponentially,
\begin{align*}
\sum_{\beta=1}^n(f_\beta(t)-y_\beta)^2
\lesssim n e^{-\frac{\la t}{2n}},
\end{align*}
which reaches $\varepsilon$ at time $t\asymp (n/\la)\ln(n/\varepsilon)$.
\end{corollary}

It is proven in \cite{du2018gradient2} that if there exists $\la^{(H)}>0$,
\begin{align*}
\la_{\min}\left[G_0^{(H)}(x_\al, x_\beta)\right]_{1\leq \al, \beta\leq n}\geq \la^{(H)},
\end{align*} 
then for $m\geq \fC (n/\la^{(H)})^4$ the gradient descent finds a global minimum. Corollary \ref{e:eigasup} improves this result in two ways: (i) We improve the quartic dependence of $n$ to a cubic dependence.
(ii) We recall that $K_t^{(2)}=\sum_{\ell=1}^{H+1}G_0^{(\ell)}$, and those kernels $G_0^{(\ell)}$ are all non-negative definite. The smallest eigenvalue of $K_0^{(2)}$ is typically much bigger than that of $G_0^{(H)}$, i.e., $\la\gg \la^{(H)}$.
Moreover, since $K_t^{(2)}$ is a sum of $H+1$ non-negative definite operators, we expect that $\la$ gets larger, if the depth $H$ is larger.

The NTH, i.e., \eqref{e:dynamic} and \eqref{e:dynamicr}, is just an infinite sequence of relationship. It cannot be used to determine the dynamic of NTK. However, thanks to the a priori estimates of the higher order kernels \eqref{e:tprior2}, it holds that with high probability $\|K_t^{(p+1)}\|_\infty\lesssim (\ln m)^\fC/m^{p/2}$. The derivative $\del_t K_t^{(p)}$ is an expression involves the higher order kernel $K_t^{(p+1)}$, which is small provided that $p$ is large enough. Therefore, we can approximate the original NTH \eqref{e:dynamicr} by simply setting $\del_t K_t^{(p)}=0$. In this way, we obtain the following truncated hierarchy of ordinary differential equations of $p$ levels, which we call {the truncated NTH}, 
\begin{align}\begin{split}\label{e:truncdynamicr}
&\del_t\tilde f_\al(t)=-\frac{1}{n}  \sum_{\beta=1}^n \tilde K^{(2)}_t(x_\al, x_{\beta})(\tilde f_\beta(t)-y_\beta),\\
&\del_t\tilde K_t^{(r)}( x_{\al_1}, x_{\al_2}, \cdots, x_{\al_{r}})=-\frac{1}{n} \sum_{\beta=1}^n \tilde K_t^{(r+1)}(x_{\al_1}, x_{\al_2}, \cdots, x_{\al_r}, x_\beta)(\tilde f_\beta(t)-y_\beta), \quad 2\leq r\leq p-1,\\
&\del_t\tilde K_t^{(p)}(x_{\al_1}, x_{\al_2}, \cdots, x_{\al_{p}})=0.
\end{split}\end{align}
where 
\begin{align*}
\tilde f_\beta(0)= f_\beta(0), \quad \beta=1,2,\cdots, n, \quad \tilde K_0^{(r)}=K_0^{(r)},\quad r=2,3,\cdots, p.
\end{align*}

In the following theorem, we show this system of truncated equations \eqref{e:truncdynamicr} approximates the dynamic of the NTK up to arbitrary precision, provided that $p$ is large enough.


\begin{theorem}\label{t:main2}
Under Assumptions \ref{a:sigmaasup} and \ref{a:nonlinear}, we take an even $p$ and further assume that  
\begin{align*}
\la_{\min} \left[K_0^{(2)}(x_\al, x_\beta)\right]_{1\leq \al, \beta\leq n}\geq \la.
\end{align*}
Then there exist  constants $\fc, \fC,\fC'>0$ such that for t
\begin{align}\label{time}
t\leq \min\{\fc \sqrt{\la m/n}/(\ln m)^{\fC}, m^{\frac{p}{2(p+1)}}/(\ln m)^{\fC'}\}, 
\end{align} 
the dynamic \eqref{e:dynamic} can be approximated by the truncated dynamic \eqref{e:truncdynamicr}, 
\begin{align}\label{e:L2er}
\left(\sum_{\beta=1}^n(f_\beta(t)-\tilde f_\beta(t))^2\right)^{1/2} \lesssim \frac{(1+t)t^{p-1}\sqrt n}{m^{p/2}} \min\left\{t, \frac{n}{\la}\right\},
\end{align}
and 
\begin{align}\label{e:kerneler}
|K_t^{(2)}(x_\al,x_\beta)-\tilde K_t^{(2)}(x_\al,x_\beta)|\lesssim\frac{(1+t)t^{p-1}}{m^{p/2}}\left(1+\frac{(1+t)t(\ln m)^{\fC}}{m}\min\left\{t, \frac{n}{\la}\right\}\right).
\end{align}
\end{theorem}

We remark that the error terms, i.e., the righthand sides of \eqref{e:L2er} and \eqref{e:kerneler} can be arbitrarily small, provided that $p$ is large enough. In other words, if we the $p$ large enough, the truncated NTH \eqref{e:truncdynamicr} can approximate the original dynamic \eqref{e:dynamic}, \eqref{e:dynamicr} up to any precision provided that the time 
constraint \eqref{time} is satisfied.
 Now if we take $t\asymp (n/\la)\ln(n/\varepsilon)$,
so that Corollary \ref{c:zeroloss} guarantees the convergence of the dynamics. Consider two special cases: 
{(i)} If we take $p=2$, then the error in \eqref{e:L2er} is $\OO(n^{7/2}\ln(n/\varepsilon)^3/\la^3m)$ when $t\asymp (n/\la)\ln(n/\varepsilon)$, which is negligible provided that the width $m$ is much bigger than $n^{7/2}$. We conclude that if $m$ is much bigger than $n^{7/2}$, the truncated NTH gives a complete description of the original dynamic of the NTK up to the equilibrium. The condition  that $m$ is much bigger than $n^{7/2}$ is better than the previous best available one which requires $m\gtrsim n^4$. 
{(ii)} If we take $p=3$, then the error in \eqref{e:L2er} is $\OO(n^{9/2}\ln(n/\varepsilon)^4/\la^4 m^{3/2})$  when $t\asymp (n/\la)\ln(n/\varepsilon)$, which is negligible provided that the width $m$ is much bigger than $n^{3}$. We conclude that if $m$ is much bigger than $n^{3}$,  the truncated NTH gives a complete description of the original dynamic of the NTK up to the equilibrium. Finally, we note that the estimates in Theorem \ref{t:main2} clearly improved for smaller $t$.

The previous convergence theory of overparametrized neural networks works only for very  wide neural networks, i.e., $m\gtrsim n^3$. 
For any width (not necessary that $m\gtrsim n^3$), Theorem \ref{t:main2} guarantees that the truncated NTH approximates the training dynamics of deep neural networks.
The effect of the width appears in the approximation time and the error terms,  \eqref{e:L2er} and \eqref{e:kerneler}, i.e., the wider the neural networks are, the truncated dynamic \eqref{e:truncdynamicr} approximates the training dynamic for longer time and the approximation error is smaller. We recall from \eqref{e:defK2} that the NTK is the sum of $H+1$ non-negative definite operators, $K_t^{(2)}=\sum_{\ell=1}^{H+1}G_t^{(\ell)}$. We expect that $\la$ gets bigger, if the depth $H$ is larger. Therefore, large width and depth
makes the  truncated dynamic \eqref{e:truncdynamicr} a better approximation.

Thanks to Theorem \ref{t:main2}, the truncated NTH \eqref{e:truncdynamicr} provides a good approximation for the evolution of the NTK. The truncated dynamic can be used to predict the output of new data points. Recall that the training data are $\{(x_\beta, y_\beta)\}_{1\leq \beta\leq n}\subset \bR^d\times \bR$. The goal is to predict the output of a new data point $x$. To do this, we can first use the truncated dynamic to solve for the approximated outputs $\{\tilde f_\beta(t)\}_{1\leq \beta\leq n}$. Then the prediction on the new test point $x\in \bR^d$ can be estimated 
by sequentially solving the higher order kernels $\tilde K_t^{(p)}(x,\cX^{p-1}), \tilde K_t^{(p-1)}(x,\cX^{p-2}),\cdots, \tilde K_t^{(2)}(x,\cX)$ and $\tilde f_x(t)$,
\begin{align}\begin{split}\label{e:preddynamic}
&\del_t\tilde f_x(t)=-\frac{1}{n}  \sum_{\beta=1}^n \tilde K^{(2)}_t(x, x_{\beta})(\tilde f_\beta(t)-y_\beta),\\
&\del_t\tilde K_t^{(r)}(x, x_{\al_1}, x_{\al_2}, \cdots, x_{\al_{r-1}})=-\frac{1}{n} \sum_{\beta=1}^n \tilde K_t^{(r+1)}(x, x_{\al_1}, x_{\al_2}, \cdots, x_{\al_r-1}, x_\beta)(\tilde f_\beta(t)-y_\beta), \quad 2\leq r \leq p-1,\\
&\del_t\tilde K_t^{(p)}(x, x_{\al_1}, x_{\al_2}, \cdots, x_{\al_{p-1}})=0.
\end{split}\end{align}

\section{Technique overview}\label{s:outline}

We recall the NTK from \eqref{e:defK2}, \begin{align*}\begin{split}
& K^{(2)}_t(x_\al,x_\beta)= \langle x_\al^{(H)}, x_\beta^{(H)}\rangle+\\
&+ \sum_{\ell=1}^H   \left\langle
\sigma'_\ell(x_\al)\frac{(W_t^{(\ell+1)})^\top}{\sqrt m}\cdots \sigma'_{H}(x_\al)\frac{a_t}{\sqrt m}, \sigma'_\ell(x_\beta)\frac{(W_t^{(\ell+1)})^\top}{\sqrt m}\cdots \sigma'_{H}(x_\beta)\frac{a_t}{\sqrt m}  \right\rangle \langle x^{(\ell-1)}_\al, x^{(\ell-1)}_\beta \rangle.
\end{split}\end{align*}
The kernel $K_t^{(2)}(\cdot, \cdot)$ is a sum of terms, which are product of inner products of vectors involving the quantities $a_t$, $W_t^{(\ell)}$, $x^{(\ell)}_\al$ and $\sigma'_\ell(x_\al)$. To compute the derivatives of $K_t^{(2)}(\cdot, \cdot)$, we need the following ordinary differential equations derived by using \eqref{e:derW}, \eqref{e:dera} and the chain rule, which characterize the dynamics of 
$a_t$, $W_t^{(\ell)}$, $x^{(\ell)}_\al$ and $\sigma_\ell^{(r)}(x_\al)$ along the gradient flow.
\begin{align*}\begin{split}
&\del_t a_t =-\frac{1}{n}\sum_{\beta=1}^n x_\beta^{(H)}(f_\beta(t)-y_\beta),\\
&\del_t \frac{W_t^{(\ell)}}{\sqrt m}=-\frac{1}{n}\sum_{\beta=1}^n\diag\left(\sigma_\ell'(x_\beta) \frac{(W_t^{(\ell+1)})^\top}{\sqrt m}\cdots \sigma_H'(x_\beta) \frac{a_t}{\sqrt m}\right)\frac{\bm 1}{\sqrt m} \otimes (x^{(\ell-1)}_\beta)^\top (f_\beta(t)-y_\beta),\\
&\del_t \frac{(W_t^{(\ell)})^\top}{\sqrt m}=-\frac{1}{n}\sum_{\beta=1}^n \frac{1}{\sqrt m}x^{(\ell-1)}_\beta\otimes \left(\frac{a_t^\top}{\sqrt m}\sigma_H'(x_\beta) \cdots \frac{W_t^{(\ell+1)}}{\sqrt m} \sigma_\ell'(x_\beta)\right)(f_\beta(t)-y_\beta),\\
&\del_t x^{(\ell)}_\al=\sum_{k=1}^\ell-\frac{1}{n}\sum_{\beta=1}^n\diag\left(\sigma'_\ell(x_\al) \frac{W_t^{(\ell)}}{\sqrt m}\cdots \frac{W_t^{(k+1)}}{\sqrt m}\sigma_k'(x_\al) \sigma_k'(x_\beta)\frac{(W_t^{(k+1)})^\top}{\sqrt m}\cdots \sigma_H'(x_\beta) \frac{a_t}{\sqrt m}\right)\\
&\phantom{{}\del_t x^{(\ell)}_\al={}}\times\frac{\bm 1}{\sqrt m} \langle x_\al^{(k-1)}, x_\beta^{(k-1)}\rangle(f_\beta(t)-y_\beta),\\
&\del_t \sigma_\ell^{(r)}(x_\al)=\sigma^{(r+1)}_\ell(x_\al)\diag(\del_t(W^{(\ell)}_tx_\al^{(\ell-1)}))\\
&=-\frac{1}{n}\sum_{\beta=1}^n\sigma^{(r+1)}_\ell(x_\al) \diag\left(\sigma'_\ell(x_\beta)\frac{(W_t^{(\ell+1)})^\top}{\sqrt m}\cdots \sigma_H'(x_\beta) \frac{a_t}{\sqrt m}\right) \langle x_\al^{(\ell-1)}, x_\beta^{(\ell-1)}\rangle (f_\beta(t)-y_\beta)\\
&\phantom{{}={}}+\sum_{k=1}^{\ell-1}-\frac{1}{n}\sum_{\beta=1}^n\sigma^{(r+1)}_\ell(x_\al)\diag\left(\frac{W_t^{(\ell)}}{\sqrt m}\sigma_{\ell-1}'(x_\al)\cdots \frac{W_t^{(k+1)}}{\sqrt m}\sigma_k'(x_\al)\sigma'_k(x_\beta)\frac{(W_t^{(k+1)})^\top}{\sqrt m}\cdots \sigma_H'(x_\beta) \frac{a_t}{\sqrt m}\right)\\
&\phantom{{}={}}\times \langle x_\al^{(k-1)}, x_\beta^{(k-1)}\rangle (f_\beta(t)-y_\beta).
\end{split}\end{align*}
%
%
%
%
We remark that the  $k=\ell$ term on the right hand side of  the expression in $\del_t x_\al^{(k)}$ is
\begin{align*}
&\del_t x^{(k)}_\al=-\frac{1}{n}\sum_{\beta=1}^n\diag\left(\sigma'_\ell(x_\al)  \sigma_\ell'(x_\beta)\frac{(W_t^{(\ell+1)})^\top}{\sqrt m}\cdots \sigma_H'(x_\beta) \frac{a_t}{\sqrt m}\right)\frac{\bm 1}{\sqrt m} \langle x_\al^{(k-1)}, x_\beta^{(k-1)}\rangle(f_\beta(t)-y_\beta).
\end{align*}
 All other cases with  $k< \ell$ can be read clearly from the expression of  $\del_t x_\al^{(k)}$ given above.

Using the chain rule and the above expressions, the derivative of $K_t^{(2)}(\cdot, \cdot)$ is given by
\begin{align*}
&\del_tK_t^{(2)}(x_{\al_1},x_{\al_2})=-\frac{1}{n} \sum_{\beta=1}^n K^{(3)}_t(x_{\al_1}, x_{\al_2}, x_\beta)(f(x_\beta, \theta_t)-y_\beta),
\end{align*}
where $K^{(3)}_t(x_{\al_1}, x_{\al_2}, x_\beta)$ is the sum of all the possible terms from $K_t^{(2)}(x_{\al_1},x_{\al_2})$ by performing one of the following replacement:
\begin{align}\begin{split}\label{e:replace}
&a_t\rightarrow x_\beta^{(H)},\\
&\frac{W_t^{(\ell)}}{\sqrt m}\rightarrow \diag\left(\sigma_\ell'(x_\beta) \frac{(W_t^{(\ell+1)})^\top}{\sqrt m}\cdots \sigma_H' (x_\beta)\frac{a_t}{\sqrt m}\right)\frac{\bm 1}{\sqrt m} \otimes (x_\beta^{(\ell-1)})^\top, \\
& \frac{(W_t^{(\ell)})^\top}{\sqrt m}\rightarrow  \frac{1}{\sqrt m}x_\beta^{(\ell-1)}\otimes \left(\frac{a_t^\top}{\sqrt m}\sigma_H'(x_\beta) \cdots \frac{W_t^{(\ell+1)}}{\sqrt m} \sigma_\ell'(x_\beta)\right),\\
&x^{(\ell)}_\alpha\rightarrow \sum_{k=1}^\ell \diag\left(\sigma'_\ell(x_\alpha) \frac{W_t^{(\ell)}}{\sqrt m}\cdots \frac{W_t^{(k+1)}}{\sqrt m}\sigma_k'(x_\alpha) \sigma_k'(x_\beta)\frac{(W_t^{(k+1)})^\top}{\sqrt m}\cdots \sigma_H'(x_\beta) \frac{a_t}{\sqrt m}\right)\frac{\bm 1}{\sqrt m} \langle x_\alpha^{(k-1)}, x_\beta^{(k-1)}\rangle,\\
&\sigma_\ell^{(r)}(x_\al)\rightarrow\sigma^{(r+1)}_\ell(x_\al)\diag\left(\sigma'_\ell(x_\beta)\frac{(W_t^{(\ell+1)})^\top}{\sqrt m}\cdots \sigma_H'(x_\beta) \frac{a_t}{\sqrt m}\right) \langle x_\al^{(\ell-1)}, x_\beta^{(\ell-1)}\rangle \\
&+\sum_{k=1}^{\ell-1}\sigma^{(r+1)}_\ell(x_\al)\diag\left(\frac{W_t^{(\ell)}}{\sqrt m}\sigma_{\ell-1}'(x_\al)\cdots \frac{W_t^{(k+1)}}{\sqrt m}\sigma_k'(x_\al)\sigma'_k(x_\beta)\frac{(W_t^{(k+1)})^\top}{\sqrt m}\cdots \sigma_H'(x_\beta) \frac{a_t}{\sqrt m}\right) \langle x_\al^{(k-1)}, x_\beta^{(k-1)}\rangle,
\end{split}\end{align}
with $\al=\al_1,\al_2$, where $\bm1=(1,1,\cdots, 1)^\top \in \bR^m$.
By the same reasoning, the derivative of $K_t^{(r)}$ is given by 
\begin{align*}
&\del_tK_t^{(r)}(x_{\al_1},x_{\al_2},\cdots, x_{\al_r})=-\frac{1}{n} \sum_{\beta=1}^n K^{(r+1)}_t(x_{\al_1}, x_{\al_2},\cdots, x_{\al_r} x_\beta)(f(x_\beta, \theta_t)-y_\beta),
\end{align*}
where $K^{(r+1)}_t(x_{\al_1}, x_{\al_2}, \cdots, x_{\al_r}, x_\beta)$  is the sum of all the possible terms from $K_t^{(r)}(x_{\al_1},x_{\al_2}, \cdots, x_{\al_r})$ by performing any of the replacements in \eqref{e:replace} with $\al=\al_1,\al_2, \cdots, \al_{r}$.

The followings are some examples of terms in $K^{(3)}_t(x_{\al_1}, x_{\al_2}, x_{\al_3})$ 
\begin{align*}
&\phantom{{}={}}\left\langle 
\sigma_{H-1}'(x_{\al_1})\frac{(W_t^{(H)})^\top}{\sqrt m} \sigma_{H}^{(2)}(x_{\al_1})\diag\left(\sigma_H'(x_{\al_3})\frac{a_t}{\sqrt m}\right)\frac{a_t}{\sqrt m},
\sigma_{H-1}'(x_{\al_2})\frac{(W_t^{(H)})^\top}{\sqrt m} \sigma_{H}'(x_{\al_2})\frac{a_t}{\sqrt m}
\right\rangle\\
&\times\langle x_{\al_1}^{(H-2)}, x_{\al_2}^{(H-2)}\rangle
\rangle\langle x_{\al_1}^{(H-1)}, x_{{\al_3}}^{(H-1)}\rangle;
\quad
\frac{1}{\sqrt m}\left\langle \sigma_H'(x_{\al_1})\frac{a_t}{\sqrt m}, \sigma_H'(x_{\al_3})\frac{a_t}{\sqrt m}\right\rangle\\
&\times\left\langle 
\sigma_{H-1}'(x_{\al_1})x_{{\al_3}}^{(H-1)},
\sigma_{H-1}'(x_{\al_2})\frac{(W_t^{(H)})^\top}{\sqrt m} \sigma_{H}'(x_{\al_2})\frac{a_t}{\sqrt m}
\right\rangle
 \langle x^{(H-2)}_{\al_1},x^{(H-2)}_{\al_2}\rangle.
\end{align*}
In general, from the construction, the summands appearing in $K_t^{(r)}(x_{\al_1}, x_{\al_2}, \cdots, x_{\al_r})$ are product of inner products of vectors obtained in the following way: starting from one of the vectors
\begin{align}\label{e:initterm}
\frac{a_t}{\sqrt m},\quad \frac{\bm 1}{\sqrt m}, \quad\{x^{(1)}_\beta,  x^{(2)}_\beta, \cdots, x^{(H)}_\beta\}_{\beta\in\{\al_1,\al_2,\cdots,\al_r\}},
\end{align}
\begin{enumerate}[(i)]
\item multiply one of the matrices 
\begin{align}\label{e:mult1}
\left\{\frac{W_t^{(2)}}{\sqrt m},\frac{(W_t^{(2)})^\top}{\sqrt m},\cdots, \frac{W_t^{(H)}}{\sqrt m},\frac{(W_t^{(H)})^\top}{\sqrt m}\right\},
\{\sigma_1'(x_\beta), \sigma_2'(x_\beta), \cdots, \sigma_H'(x_\beta)\}_{\beta\in\{\al_1,\al_2,\cdots,\al_r\}};
\end{align}
\item multiply one of the matrices
\begin{align}\label{e:mult2}
 \diag(\cdots),\quad \sigma^{(s)}(x_\beta)\underbrace{\diag(\cdots)\cdots\diag(\cdots)}_{s-1\text{ terms}},\quad s\geq 2
\end{align}
where $\diag(\cdots)$ is the diagonalization of a vector obtained by recursively using 1) and 2).
\end{enumerate}

To describe the vectors appearing in $K_t^{(r)}(x_{\al_1}, x_{\al_2}, \cdots, x_{\al_r})$ in a formal way, we need to introduce some more notations. We denote $\fD_0$ the set of expressions in the following form 
\begin{align}\label{e:defD1}
\fD_0: = \{ \sfe_s \sfe_{s-1}\cdots \sfe_{1}\sfe_0:   0\leq s\leq 4H-3 \}, 
\end{align}
where 
$\sfe_j$ is chosen from the following sets:  
\begin{align*}
\sfe_0\in \left\{a_t, \{\sqrt m x^{(1)}_\beta, \sqrt m x^{(2)}_\beta, \cdots, \sqrt m x^{(H)}_\beta\}_{1\leq \beta\leq n}\right\}
\end{align*}
and for $1\leq j\leq s$,
\begin{align*}
\sfe_j\in \left\{\left\{\frac{W_t^{(2)}}{\sqrt m},\frac{(W_t^{(2)})^\top}{\sqrt m},\cdots, \frac{W_t^{(H)}}{\sqrt m},\frac{(W_t^{(H)})^\top}{\sqrt m}\right\} , \{\sigma_1'(x_\beta), \sigma_2'(x_\beta), \cdots, \sigma_H'(x_\beta)\}_{1\leq \beta\leq n}\right\}.
\end{align*}
We remark that from expression \eqref{e:defK2}, each summand in  $K_t^{(2)}(x_{\al_1}, x_{\al_2})$ is of the form  
\begin{align*}
\frac{\langle \sfv_1(t), \sfv_2(t)\rangle}{m}, \quad \frac{\langle \sfv_1(t), \sfv_2(t)\rangle}{m}\frac{\langle \sfv_3(t), \sfv_4(t)\rangle}{m},
\end{align*}
where $\sfv_1(t), \sfv_2(t),\sfv_3(t), \sfv_4(t)\in \fD_0$. But the set $\fD_0$  contains more terms than those appearing in $K_t^{(2)}(\cdot, \cdot)$.
Given that we have constructed $\fD_0, \fD_1, \cdots, \fD_r$, we denote $\fD_{r+1}$ the set of expressions in the following form
\begin{align}\label{e:term}
\fD_{r+1}\deq\{\sfe_s \sfe_{s-1} \cdots \sfe_1\sfe_0: 0\leq s\leq 4H-3\},
\end{align}
where $\sfe_j$ is chosen from the following sets   (notice that  we have included  $\bm1 $  in the following set, 
which does not appear in the definition of $\fD_0$): 
\begin{align*}
\sfe_0\in \left\{a_t, \bm1,\{\sqrt m x^{(1)}_\beta, \sqrt m x^{(2)}_\beta, \cdots, \sqrt m x^{(H)}_\beta\}_{1\leq \beta\leq n}\right\},
\end{align*}
and for $1\leq j\leq s$, $\sfe_j$ belongs to one of the sets
\begin{align*}\begin{split}
& \left\{\left\{\frac{W_t^{(2)}}{\sqrt m},\frac{(W_t^{(2)})^\top}{\sqrt m},\cdots, \frac{W_t^{(H)}}{\sqrt m},\frac{(W_t^{(H)})^\top}{\sqrt m}\right\} , \{\sigma_1'(x_\beta), \sigma_2'(x_\beta), \cdots, \sigma_H'(x_\beta)\}_{1\leq \beta\leq n}\right\},\\
&\left\{\diag(\sfd),\quad \sfd\in \fD_0\cup\fD_1\cup\cdots\cup \fD_r\right\},\\
\begin{split}
&\left\{\sigma_\ell^{(u+1)}(x_\beta)\diag(\sfd_1)\diag(\sfd_2)\cdots \diag(\sfd_u): 1\leq \ell\leq H, \right. \\
&\left.\phantom{\sfe_j\in \left\{\sigma_\ell^{(u+1)}\right\}}1\leq \beta\leq n, 1\leq u\leq r, \sfd_1,\sfd_2,\cdots,\sfd_u\in \fD_0\cup \fD_1\cup\cdots\cup \fD_r\right\}.
\end{split}
\end{split}\end{align*}
Moreover, 
the total number of $\diag$ operations in the expression $\sfe_s\sfe_{s-1}\cdots \sfe_1\sfe_0\in\fD_{r+1}$   is exactly $r+1$.  We remark that if $\sfd\in \fD_s$, then it contains $s$ $\diag$ operations.  On the other hand, by definition, we view   $\diag(\sfd)$ as an element with  $s+1$ $\diag$ operations because the $\diag$ in $\diag(\sfd)$ counted as one $\diag$ operation.

The kernel $K_t^{(3)}(x_{\al_1}, x_{\al_2},  x_{\al_3})$ is obtained from $K_t^{(2)}(x_{\al_1},x_{\al_2})$ by the replacements \eqref{e:replace} and taking $\al=\al_1,\al_2$ and $\beta=\al_3$. The summands in $K_t^{(3)}(x_{\al_1}, x_{\al_2}, x_{\al_3})$ are of the forms
\begin{align}\label{e:terms}
\frac{1}{\sqrt m}\frac{\langle \sfv_1(t),\sfv_2(t)\rangle}{m},\quad 
\frac{1}{\sqrt m}\frac{\langle \sfv_1(t),\sfv_2(t)\rangle}{m}\frac{\langle \sfv_3(t),\sfv_4(t)\rangle}{m},\quad \frac{1}{\sqrt m}\frac{\langle \sfv_1(t),\sfv_2(t)\rangle}{m}\frac{\langle \sfv_3(t),\sfv_4(t)\rangle}{m}\frac{\langle \sfv_5(t),\sfv_6(t)\rangle}{m},
\end{align}
where $\sfv_1(t),\sfv_2(t),\cdots,\sfv_6(t)\in \fD_0\cup \fD_1$.
The first two terms in \eqref{e:terms} are obtained from using the replacements for $a_t$, and the last two  terms in \eqref{e:terms} are obtained from using the replacements for $W_t^{(\ell)}/\sqrt m$, $(W_t^{(\ell)})^\top/\sqrt m$, $x_\al^{(\ell)}$ and $\sigma_\ell^{(r)}(x_\al)$. 
More generally, we will show that each summand in $K_t^{(r)}(x_{\al_1}, x_{\al_2},\cdots, x_{\al_r})$ is of the form
\begin{align}\label{e:form}
\frac{1}{m^{r/2-1}}\prod_{j=1}^s \frac{\langle \sfv_{2j-1}(t),\sfv_{2j}(t)\rangle}{m}, \quad 1\leq s\leq r,\quad \sfv_1(t), \sfv_2(t),\cdots, \sfv_{2s}(t)\in \fD_0\cup \fD_1\cup \cdots \cup \fD_{r-2}.
\end{align}

%


The initial value $K_0^{(r)}(x_{\al_1}, x_{\al_2}, \cdots, x_{\al_r})$ can be estimated  by successively conditioning based on the depth of the neural network. A convenient scheme is given by the tensor program \cite{DBLP:journals/corr/abs-1902-04760},  
which was developed to characterize the scaling limit of neural network computations. In Appendix \ref{s:initial}, we show at time $t=0$, those vectors $\sfv_j(0)$ in \eqref{e:form} are combinations of projections of independent Gaussian vectors. As a consequence, we have that $\langle \sfv_{2j-1}(0), \sfv_{2j}(0)\rangle/m$ concentrates around certain constant with high probability. So does the product $\prod_{j=1}^s \langle \sfv_{2j-1}(0), \sfv_{2j}(0)\rangle/m$. This gives the claim \eqref{e:tprior1}. 

In Appendix \ref{s:prior}, we consider the quantity:
\begin{align*}
 \xi(t)=\max\{\|\sfv_j(t)\|_\infty: \sfv_j(t)\in \fD_0\cup\fD_1\cup\cdots\cup \fD_{p-1}\}.
 \end{align*}
Again using the tensor program, we show that with high probability $\|\sfv_j(0)\|_\infty\lesssim (\ln m)^\fC$. This gives the estimate of $\xi(t)$ at $t=0$. Next we show that the $(p+1)$-th derivative of $\xi(t)$ can be controlled by itself. This gives a self-consistent differential equation of $\xi(t)$:
\begin{align}
\del^{(p+1)}_{t}\xi(t)\lesssim \frac{\xi(t)^{2p}}{m^{p/2}}.
\end{align}
 Combining with the initial estimate of $\xi(t)$, it follows that for time ${0\leq  t\leq m^{\frac{p}{2(p+1)}}}/(\ln m)^{\fC'}$, it holds that $\xi(t)\lesssim (\ln m)^{\fC}$. Especially $\|\sfv_j(t)\|_\infty\lesssim (\ln m)^\fC$. Then the claim \eqref{e:tprior2} in Theorem \ref{t:main1} follows.

Thanks to the a priori estimate \eqref{e:tprior2}, we show that along the continuous time gradient descent, the higher order kernels $K_t^{(r)}$ vary slowly. We prove Corollary \ref{c:change} and \ref{c:zeroloss}, and Theorem \ref{t:main2} in Appendix \ref{s:cproof} by a Gr{\" o}nwall type argument.

\section{Discussion and future directions}
In this paper, we study the continuous time gradient descent (gradient flow) of deep fully-connected neural networks. We show that the training dynamic is given by a data dependent infinite hierarchy of ordinary differential equations, i.e., the NTH. We also show that this dynamic of the NTK can be approximated by a finite truncated dynamic up to any precision. This description makes it possible to directly study the change of the NTK for deep neural networks. Here we list some future directions.

Firstly, we mainly study deep fully-connected neural networks here, we believe the same statements can be proven for convolutional and residual neural networks.

Secondly, in this paper, for simplicity, we focus on the continuous time gradient descent. Our approach developed here can be generalized to analyze discrete time gradient descent. We elaborate the main idea here. The discrete time gradient descent is given by
\begin{align*}
\theta_{t+1}=\theta_t-\eta \nabla_\theta L(\theta_t)
=\theta_t-\frac{\eta}{n}\sum_{\beta=1}^n \nabla_\theta f_\beta(t)(f_\beta(t)-y_\beta),
\end{align*}
where $\eta$ is the learning rate. We write the NTK as ${\mathcal K}^{(2)}(x_\al, x_\beta;\theta_t)$ to make the dependence on $\theta_t$ explicit. To estimate the NTK ${\mathcal K}^{(2)}(x_{\al_1}, x_{\al_2};\theta_{t+1})$ at time $t+1$, we use the taylor expansion,
\begin{align}\begin{split}\label{e:taylor}
&{\mathcal K}^{(2)}(x_{\al_1}, x_{\al_2};\theta_{t+1})
\approx {\mathcal K}^{(2)}(x_{\al_1}, x_{\al_2};\theta_{t})+\sum_{r=3}^{p-1}\frac{(-\eta)^r}{n^r}\sum_{1\leq \beta_1,\beta_2,\cdots,\beta_{r-2}\leq n}\\
&{\mathcal K}^{(r)}(x_{\al_1}, x_{\al_2}, x_{\beta_1}, \cdots,x_{\beta_{r-2}};\theta_{t})
(f_{\beta_1}(t)-y_{\beta_1})\cdots (f_{\beta_{r-2}}(t)-y_{\beta_{r-2}}),
\end{split}\end{align}
where the higher order kernels $\mathcal K^{(r)}$ are given by
\begin{align*}
{\mathcal K}^{(r)}(x_{\al_1}, x_{\al_2}, x_{\beta_1}, \cdots,x_{\beta_{r-2}};\theta_t)
=\nabla_\theta^{(r-2)}{\mathcal K}^{(2)}(x_{\al_1}, x_{\al_2};\theta_t)(\nabla_\theta f_{\beta_1}(t), \nabla_\theta f_{\beta_2}(t),\cdots, \nabla_\theta f_{\beta_{r-2}}(t)).
\end{align*}
A similar argument as for \eqref{e:tprior2} can be used to derive the a priori estimates of these kernels $\mathcal K^{(r)}$. We expect to have that $\|\mathcal K^{(r)}\|_\infty\lesssim (\ln m)^\fC/m^{r/2-1}$ with high probability with respect to the random initialization. Therefore the righthand side of \eqref{e:taylor} gives an approximation of the NTK ${\mathcal K}^{(2)}(x_{\al_1}, x_{\al_2};\theta_{t+1})$ at time $t+1$ up to arbitrary precision, provided that $p$ is large enough. This gives a description of the NTK dynamics under discrete time gradient descent.

%

%

\bibliography{all.bib}{}
\bibliographystyle{abbrv}

\appendix

\section{Initial Estimates}
\label{s:initial}

We have derived the dynamic \eqref{e:dynamicr} of the NTK in Section \ref{s:outline}. The kernel $K^{(r+1)}_t(x_{\al_1}, x_{\al_2}, \cdots, x_{\al_r}, x_\beta)$  is the sum of all the possible terms from $K_t^{(r)}(x_{\al_1},x_{\al_2}, \cdots, x_{\al_r})$ by performing any of the replacements in \eqref{e:replace} with $\al=\al_1,\al_2, \cdots, \al_{r}$. We recall the sets $\fD_r$ from Section \ref{s:outline}, which are constructed recursively. Each vector in $\fD_r$ contains exact $r$ $\diag$ operations. 
We have the following proposition on the structures of vectors in $\fD_r$.
\begin{proposition}\label{p:fDstruc}
Given any expression $\sfv(t)\in \fD_r$ with some $r\geq 0$, new expressions obtained from $\sfv(t)$ by performing one of the replacements in \eqref{e:replace} are sum of terms of the following forms:
\begin{itemize}
\item $\sfv_1(t)$ with $\sfv_1(t)\in \fD_{r}\cup \fD_{r+1}$;
\item $  \frac { \sfv_1(t)} {\sqrt m}  \frac { \langle \sqrt m x_{\al}^{(k-1)},\sqrt m x_{\beta}^{(k-1)}\rangle} m
$ with $1\leq k\leq H$ and $\sfv_1(t)\in \fD_{r+1}$;
\item $ \frac { \sfv_1(t)} { \sqrt m} \frac {  \langle \sqrt m x_{\beta}^{(\ell)}, \sfv_2(t)\rangle} m$ with $1\leq \ell \leq H$ and $\sfv_1(t)\in \fD_{r-s+1}$ and $\sfv_2(t)\in \fD_{s}$ for some $s\geq 1$;
\item $\frac { \sfv_1(t)} { \sqrt m} \frac {  \left\langle \sigma_\ell'(x_{\beta}) (W_t^{(\ell+1)})^\top/\sqrt m \cdots 
\sigma_H' (x_{\beta})a_t, \sfv_2(t)\right\rangle} m$ with $1\leq \ell \leq H$, $\sfv_1(t)\in \fD_{r-s}$ and $\sfv_2(t)\in \fD_{s}$ for some $s\geq 1$.
\end{itemize}
\end{proposition}
We remark that the time $t$ in Proposition \ref{p:fDstruc}  is only a parameter and this proposition does not involve dynamics.

\begin{proof}[Proof of Proposition \ref{p:fDstruc}]
By performing the replacement for $a_t$, the new expression is given by
$
\sfv_1(t)/\sqrt m,
$
with $\sfv_1(t)\in \fD_r$.

By performing the replacement for $x_\al^{(\ell)}$, we get a sum of $\ell$ terms. Each of them is of the form
$\sfv_1(t)/\sqrt m\langle \sqrt m x_{\al}^{(k-1)},\sqrt m x_{\beta}^{(k-1)}\rangle/m
$ with $\sfv_1(t)$ containing one more $\diag$ operations. It is easy to check that $\sfv_1(t)\in \fD_{r+1}$.

By performing the replacement for $W_t^{(\ell)}/\sqrt m$, the new expression is given by 
\begin{align*}
\frac{ \sfv_1(t)}{\sqrt m}\frac{ \langle \sqrt m x_{\beta}^{(\ell)}, \sfv_2(t)\rangle}{m}, 
\end{align*}
with $\sfv_1(t)\in \fD_{r-s+1}$ and $\sfv_2(t)\in \fD_{s}$ for some $s\geq 1$.
%

By performing the replacement for $(W_t^{(\ell)})^\top/\sqrt m$, the new expression is given by
\begin{align*}
\frac{\sfv_1(t)}{\sqrt m}  \frac{1}{m}\left\langle \sigma_\ell'(x_{\beta}) \frac{(W_t^{(\ell+1)})^\top}{\sqrt m}\cdots \sigma_H' (x_{\beta})a_t, \sfv_2(t)\right\rangle,
\end{align*}
with $\sfv_1(t)\in \fD_{r-s}$ and $\sfv_2(t)\in \fD_{s}$ for some $s\geq 1$.

Finally, by performing the replacement for $\sigma_\ell^{(u)}(x_\al)$, we get a sum of $\ell$ terms of the form $\sfv_1(t)/\sqrt m$, with $\sfv_1(t)\in \fD_{r+1}$.

\end{proof}

As a consequence of Proposition \ref{p:fDstruc}, each summand in $K_t^{(r)}(x_{\al_1}, x_{\al_2},\cdots, x_{\al_r})$ is of the form
\begin{align}\label{e:formcc}
\frac{1}{m^{r/2-1}}\prod_{j=1}^s \frac{\langle \sfv_{2j-1}(t),\sfv_{2j}(t)\rangle}{m}, \quad 1\leq s\leq r,\quad \sfv_1(t), \sfv_2(t),\cdots, \sfv_{2s}(t)\in \fD_0\cup \fD_1\cup \cdots \cup \fD_{r-2}.
\end{align}



In the rest of this section we prove claim \eqref{e:tprior1} in Theorem \ref{t:main1}.
To evaluate $K_0^{(r)}(x_{\al_1}, x_{\al_2}, \cdots, x_{\al_r})$, we use the {tensor program} in \cite{DBLP:journals/corr/abs-1902-04760}, which was developed to characterize the scaling limit of neural network computations. We show at time $t=0$, those vectors $\sfv_j(0)$ in \eqref{e:form} are combinations of projections of independent Gaussian vectors. As a consequence, we have that $\langle \sfv_{2j-1}(0), \sfv_{2j}(0)\rangle/m$ concentrates around certain constant with high probability. So does the product $\prod_{j=1}^s \langle \sfv_{2j-1}(0), \sfv_{2j}(0)\rangle/m$. This gives the claim \eqref{e:tprior1}. 

In the next section, we consider the quantity:
\begin{align*}
 \xi(t)=\max\{\|\sfv_j(t)\|_\infty: \sfv_j(t)\in \fD_0\cup\fD_1\cup\cdots\cup \fD_{p-1}\}.
 \end{align*}
Again using the tensor program, we show that with high probability $\|\sfv_j(0)\|_\infty\lesssim (\ln m)^\fC$. This gives the estimate of $\xi(t)$ at $t=0$. Next we show that the $(p+1)$-th derivative of $\xi(t)$ can be controlled by itself. This gives a self-consistent differential equation of $\xi(t)$. Combining with the initial estimate of $\xi(t)$, it follows that for time ${0\leq  t\leq m^{\frac{p}{2(p+1)}}}/(\ln m)^{\fC'}$, it holds that $\xi(t)\lesssim (\ln m)^{\fC}$. Especially $\|\sfv_j(t)\|_\infty\lesssim (\ln m)^\fC$. Then the claim \eqref{e:tprior2} in Theorem \ref{t:main1} follows.

\begin{proposition}\label{p:Krlimit}
Under Assumptions \ref{a:sigmaasup} and \ref{a:nonlinear}, there exists a deterministic family of operators $\fK^{(r)}: \cX^r\mapsto \bR$ for $2\leq r\leq p+1$ (independent of $m$) and $\fK^{(r)}=0$ if $r$ is odd, such that with high probability with respect to the random initialization, it holds that
\begin{align}\label{e:Krlimit}
\left\|K^{(r)}_0-\frac{\fK^{(r)}}{m^{r/2-1}}\right\|_{\infty}\lesssim \frac{(\ln m)^\fC}{m^{(r-1)/2}}.
\end{align}
\end{proposition}

As we have shown in \eqref{e:formcc}, the kernel $K^{(r)}_t(x_{\al_1}, x_{\al_2}, \cdots, x_{\al_r})$ is a sum of terms in the form
\begin{align}\label{e:formcopy}
\frac{1}{m^{r/2-1}}\prod_{j=1}^s \frac{\langle \sfv_{2j-1}(t),\sfv_{2j}(t)\rangle}{m}, \quad 1\leq s\leq r,\quad \sfv_1(t), \sfv_2(t),\cdots, \sfv_{2s}(t)\in \fD_0\cup \fD_1\cup \cdots \cup \fD_{r-2}.
\end{align}
To evaluate $K_0^{(r)}(x_{\al_1}, x_{\al_2}, \cdots, x_{\al_r})$, 
we  recall the following conditioning Lemma  from \cite{DBLP:journals/corr/abs-1902-04760}.
With this lemma, 
we can keep track of vectors appearing in the expression of $K_0^{(r)}(x_{\al_1}, x_{\al_2}, \cdots, x_{\al_r})$, and their decomposition into combinations of projections of independent Gaussian vectors. 
\begin{lemma}\label{l:decompose}
Let $W\in \bR^{m\times m}$ be a matrix with random Gaussian entries $W_{ij}\sim \cN(0,c_w)$. Consider fixed matrices $Q\in \bR^{m\times q}, Y\in \bR^{m\times q}, P\in \bR^{m\times p},  X\in \bR^{m\times p}$. Then the distribution of $W$ conditioned on $Y=WQ$ and $X=W^\top P$ is 
\begin{align*}
W|_{Y=WQ, X=W^\top P}\stackrel{\text{d}}{=} E+\Pi_P^\perp\tilde W \Pi_Q^\perp,
\end{align*}
where $\tilde W$ is an independent copy of $W$,
\begin{align*}
E=YQ^++(P^+)^\top X^\top\Pi_Q^\perp=\Pi_P^\perp YQ^++(P^+)^\top X^\top,
\end{align*}
$Q^+, P^+$ are Moore-Penrose pseudoinverse of $Q,P$ respectively, and $\Pi_Q=I_m-\Pi_Q^\perp=QQ^+, \Pi_P=I_m-\Pi_P^\perp=PP^+$ are the orthogonal projection on the space spanned by the columns of $Q,P$ repsectively.
\end{lemma}

\begin{proof}[Proof of Proposition \ref{p:Krlimit}]
Without loss of generality, we simply take $x_{\al_1}=x_1, x_{\al_2}=x_2, \cdots, x_{\al_r}=x_r$. We decompose the expression of $K_0^{(r)}(x_1,x_2,\cdots,x_r)$ into sub-expressions. We denote
\begin{align*}
\sfe_0=a_0,\quad  \sfe_{r(\ell-1)+i}=W_0^{(\ell)} x^{(\ell-1)}_{i},\quad 1\leq i\leq r,\quad 1\leq \ell\leq H.
\end{align*}
In the rest of the proof, we view $\sfe_i$ as formal expressions, and we denote their values as $\val(\sfe_i)$. For the computation, to evaluate $f(x_1,\theta_0), f(x_2,\theta_0), \cdots, f(x_r,\theta_0)$, we need to sequentially evaluate the expressions $ \sfe_1,  \sfe_2,\cdots, \sfe_{rH}$. We will express the values of these expressions as combinations of Gaussian vectors in the following way. By repeatedly using Lemma \ref{l:decompose}, we have
\begin{align}\begin{split}\label{e:first}
&\val( \sfe_1)=W_0^{(1)}x_1\stackrel{d}{=}\fa_{1,1}g_1,\\
&\val( \sfe_2)=W_0^{(1)}x_2\stackrel{d}{=}\fa_{2,2}g_2+\fa_{2,1}g_{1},\\
&\cdots\cdots\\
&\val( \sfe_r)=W_0^{(1)}x_r\stackrel{d}{=}\fa_{r,r}g_r+\fa_{r,r-1}g_{r-1}+\cdots+\fa_{r,1}g_1.
\end{split}\end{align}
where $g_1, g_2,\cdots, g_r$ are independent standard Gaussian vectors in $\bR^m$; the coefficients $\fa_{i,j}$ can be computed by performing the Gram-Schmidt algorithm over the input vectors $x_1,x_2,\cdots,x_r$, which depend only on the inner products $\langle x_i, x_j\rangle$ and we call them \emph{A-variables}. In general A-variables are random variables, however in \eqref{e:first}, they are deterministic. Thanks to the Assumption \ref{a:nonlinear}, the smallest singular value of the matrix $[x_{1}, x_{2}, \cdots, x_{r}]$ is at least $\fc_r>0$, the leading coefficients $|\fa_{1,1}|, |\fa_{2,2}|, \cdots, |\fa_{r,r}|\asymp 1$. As a consequence, each of the evaluations of $\sfe_i$ for $1\leq i\leq r$ contains a new standard Gaussian vector.

For the output of the second layer, again using Lemma \ref{l:decompose}, we have
\begin{align*}\begin{split}
&\val(\sfe_{r+1})=\frac{W_0^{(2)}}{\sqrt m}\sigma(\val( \sfe_1))\stackrel{d}{=}\fa_{r+1,r+1}g_{r+1},\\
&\val( \sfe_{r+2})=\frac{W_0^{(2)}}{\sqrt m}\sigma(\val( \sfe_2))\stackrel{d}{=}\fa_{r+2,r+2}g_{r+2}+\fa_{r+2,r+1}g_{r+1},\\
&\cdots\cdots\\
&\val( \sfe_{2r})=\frac{W_0^{(2)}}{\sqrt m}\sigma(\val( \sfe_r))\stackrel{d}{=}\fa_{2r,2r}g_{2r}+\fa_{2r,2r-1}g_{2r-1}+\cdots+\fa_{2r,r+1}g_{r+1}.
\end{split}\end{align*}
where $g_{r+1}, g_{r+2},\cdots, g_{2r}$ are independent standard Gaussian vectors in $\bR^m$, which are also independent of $g_1,g_2,\cdots, g_r$; the coefficients $\fa_{i,j}$ are computed by performing the Gram-Schmidt algorithm over $x^{(1)}_1,x^{(1)}_2,\cdots,x^{(1)}_r$. In this case, the coefficients $\fa_{i,j}$ are random, which depend on the inner products $\langle x^{(1)}_i, x^{(1)}_j\rangle$. However, the inner products $\langle x^{(1)}_i, x^{(1)}_j\rangle$ 
\begin{align*}
\langle x^{(1)}_i, x^{(1)}_j\rangle
=\frac{1}{m}\langle\sigma(\fa_{i,i}g_i+\fa_{i,i-1}g_{i-1}+\cdots+\fa_{i,1}g_1) , \sigma(\fa_{j,j}g_j+\fa_{j,j-1}g_{j-1}+\cdots+\fa_{j,1}g_1)\rangle,
\end{align*}
are average of $m$ independent identically distributed quantities, each is a function of Gaussian variables. Therefore, $\langle x_i^{(1)}, x_j^{(1)}\rangle$ has a scaling limit as the width of the neural network $m\rightarrow \infty$, and strongly concentrates around this limit. In other words, with high probability we have
\begin{align}\label{e:Avar}
\lim_{m\rightarrow \infty}\fa_{i,j}=\tilde \fa_{i,j}, \quad \fa_{i,j}=\tilde \fa_{i,j}+\OO\left(\frac{(\ln m)^{\fC}}{\sqrt m}\right).
\end{align}
We will see soon, in fact, by the same reasoning, all the A-variables appearing in this section satisfy the relation \eqref{e:Avar}. Moreover, in the limit $m\rightarrow \infty$, The Gram matrix of $x_1^{(1)}, x_2^{(1)}, \cdots, x_r^{(1)}$ is full rank. Otherwise there exist constants $\la_1, \la_2,\cdots,\la_r$ such that
\begin{align}\label{e:exp}
\la_1 \sigma(\tilde \fa_{1,1}G_{1})+\la_1 \sigma(\tilde \fa_{2,2}G_{2}+\tilde\fa_{2,1}G_{1})+\cdots +\la_r \sigma(\tilde \fa_{r,r}G_r+\tilde \fa_{r,r-1}G_{r-1}+\cdots+\fa_{r,1}G_{1})=0,
\end{align}
for independent Gaussian variables $G_1,G_2,\cdots,G_r\sim \cN(0,1)$.
This is impossible, unless the expression \eqref{e:exp} is literally zero, i.e. $\la_1, \la_2,\cdots,\la_r=0$. Therefore, in the limit $m\rightarrow \infty$, The Gram matrix of $x_1^{(1)}, x_2^{(1)}, \cdots, x_r^{(1)}$ is full rank. We conclude that $|\tilde \fa_{r+1,r+1}|,|\tilde \fa_{r+2,r+2}|, \cdots, |\tilde \fa_{2r,2r}|\asymp 1$. Combining with \eqref{e:Avar}, with high probability, it holds $|\fa_{r+1,r+1}|,|\fa_{r+2,r+2}|, \cdots, |\fa_{2r,2r}|\asymp 1$. Again, each of the evaluations of $\sfe_i$ for $r+1\leq i\leq 2r$ contains a new standard Gaussian vector.

By repeating the above argument, we get that for any $1\leq i \leq rH$,
\begin{align*}
\val( \sfe_j)\stackrel{d}{=}\fa_{j,j}g_{j}+\fa_{j,j-1}g_{j-1}\cdots+\fa_{j,1}g_1,
\end{align*}
where $g_1, g_2, \cdots, g_{rH}$ are independent standard Gaussian vectors, the A-variables $\fa_{j,j},\fa_{j,j-1},\cdots,\fa_{j,1}$ concentrate around their limits, i.e. with high probability \eqref{e:Avar} holds, and $|\fa_{j,j}|\asymp 1$.

Formally as expressions, we have for $2\leq \ell\leq H$,
\begin{align*}
\left[\sfe_{r(\ell-1)+1}, \sfe_{r(\ell-1)+2},\cdots,  \sfe_{r\ell}\right]=W_0^{(\ell)}\left[\frac{\sigma( \sfe_{r(\ell-2)+1})}{\sqrt m}, \frac{\sigma( \sfe_{r(\ell-2)+2})}{\sqrt m},\cdots, \frac{\sigma( \sfe_{r(\ell-1)})}{\sqrt m}\right].
\end{align*}
To use Lemma \ref{l:decompose} in the future, we denote for $2\leq \ell \leq H$,
\begin{align*}\begin{split}
&Y_0^{(\ell)}=\val\left[\sfe_{r(\ell-1)+1},  \sfe_{r(\ell-1)+2},\cdots, \sfe_{r\ell}\right],\\
&Q_0^{(\ell)}=\val \left[\frac{\sigma( \sfe_{r(\ell-2)+1})}{\sqrt m}, \frac{\sigma( \sfe_{r(\ell-2)+2})}{\sqrt m},\cdots, \frac{\sigma( \sfe_{r(\ell-1)})}{\sqrt m}\right].
\end{split}\end{align*}
Then $Y_0^{(\ell)}=W_0^{(\ell)}Q_0^{(\ell)}$, for $2\leq \ell \leq H$.

To estimate $K^{(r)}_0(x_1,x_2,\cdots,x_r)$, we need to decompose the expression of $K^{(r)}_0(x_1,x_2,\cdots,x_r)$ into subexpressions $\sfe_{rH+1}, \sfe_{rH+2},\sfe_{rH+3},\cdots$ in the following way. Since each summand in $K^{(r)}_0(x_1,x_2,\cdots,x_r)$ is of the form \eqref{e:form}. For each of these vectors $\sfv_j(0), \sfu_j(0)$,
we evaluate it from right to left. Each time, when we need to multiply one of these matrices $W_0^{(2)}, (W_0^{(2)})^{\top},W_0^{(3)}, (W_0^{(3)})^{\top}, \cdots, W_0^{(H)},  (W_0^{(H)})^{\top}$, we add a new subexpression corresponding to the whole expression if it has not appeared before. For example we have the following expression in $K^{(2)}_0(\cdot,\cdot)$:
\begin{align}\label{e:example}
\sigma_\ell'(x)\frac{(W_0^{(\ell+1)})^\top}{\sqrt m}\cdots \sigma_H'(x)a_0.
\end{align}
We decompose it into subexpressions in the following way
\begin{align*}\begin{split}
&\sfe_{rH+1}=\frac{(W_0^{(H)})^\top}{\sqrt m}\sigma_H'(x)a_0,\\
&\sfe_{rH+2}=\frac{(W_0^{(H-1)})^\top}{\sqrt m}\sigma_{H-1}'(x)\frac{(W_0^{(H)})^\top}{\sqrt m}\sigma_H'(x)a_0,\\
&\cdots\cdots\\
&\sfe_{(r+1)H-\ell}=\frac{(W_0^{(\ell+1)})^\top}{\sqrt m}\cdots \sigma_H'(x)a_0=\eqref{e:example}.
\end{split}\end{align*}
Then for each $\sfe_j$ with $j\geq rH+1$, either $\sfe_j=(W_0^{(\ell)}/\sqrt{m})\sff_j$ or $\sfe_j=((W_0^{(\ell)})^\top/\sqrt{m})\sff_j$ for some $2\leq \ell\leq H$, and $\sff_j$ is an expression in the following form
\begin{align*}
\Mul(\sfe_0, \sfe_1,\cdots, \sfe_{j-1})=\{\text{entrywise products of } \sfe_0, \{\sigma^{(s)}(\sfe_i)\}_{0\leq s\leq r-1, 1\leq i\leq rH}, \{\sfe_i\}_{rH+1\leq i\leq j-1}\}.
\end{align*}
For $2\leq \ell\leq H$, we denote the sets
\begin{align*}
&S^{(\ell)}_\tau=\{1\leq j\leq rH+\tau: \sfe_j \text{ ends with multiplying } W_0^{(\ell)}/\sqrt m\},\\
&T^{(\ell)}_\tau=\{1\leq j\leq rH+\tau: \sfe_j \text{ ends with multiplying } (W_0^{(\ell)})^\top/\sqrt m\}.
\end{align*}
Formally as expressions, we have for $2\leq \ell\leq H$,
\begin{align*}
&\left[\sfe_j\right]_{j\in S^{(\ell)}_\tau}=W_0^{(\ell)}\left[\frac{\sff_j}{\sqrt m}\right]_{j\in S^{(\ell)}_\tau},\\
&\left[\sfe_j\right]_{j\in T^{(\ell)}_\tau}=(W_0^{(\ell)})^\top\left[\frac{\sff_j}{\sqrt m}\right]_{j\in T^{(\ell)}_\tau}.
\end{align*}
To use Lemma \ref{l:decompose} in the future, we denote for $2\leq \ell \leq H$,
\begin{align}\begin{split}\label{e:defYQXP}
&Y_\tau^{(\ell)}=\val\left[\sfe_j\right]_{j\in S^{(\ell)}_\tau},\quad Q_\tau^{(\ell)}=\val \left[\frac{\sff_j}{\sqrt m}\right]_{j\in S^{(\ell)}_\tau},\\
&X_\tau^{(\ell)}=
\val \left[\sfe_j\right]_{j\in T^{(\ell)}_\tau},\quad P_\tau^{(\ell)}=\val\left[\frac{\sff_j}{\sqrt m}\right]_{j\in T^{(\ell)}_\tau}.
\end{split}\end{align}
Then  for $2\leq \ell\leq H$,
\begin{align*}
Y_\tau^{(\ell)}=W_0^{(\ell)}Q_\tau^{(\ell)},\quad X_\tau^{(\ell)}=(W_0^{(\ell)})^{\top}P_\tau^{(\ell)}.
\end{align*}

In the following we prove by induction that
\begin{claim}\label{c:newtau}
For $\tau\geq 1$, the following holds.
\begin{enumerate}[{\rm (i)}]
\item
The limits as $m\rightarrow \infty$ of the Gram matrices of columns $Q_\tau^{(\ell)}$, and columns of $P_\tau^{(\ell)}$, as defined in \eqref{e:defYQXP}, are non-degenerate;
\item 
Let $\Mul(a_0, g_1,g_2,\cdots, g_{rH+\tau-1})$ be the set of entrywise products of 
$a_0$, $\{\sigma^{(s)}(\fa_{i,i}g_i+\fa_{i,i-1}g_{i-1}+\cdots+\fa_{i,1}g_1)\}_{0\leq s\leq r-1, 1\leq i\leq rH}$ and $\{g_i\}_{rH+1\leq i\leq rH+\tau-1}$ and $\LinMul(a_0, g_1,g_2,\cdots, g_{rH+\tau-1})$
be the set of linear combinations of $\Mul(a_0, g_1,g_2,\cdots, g_{rH+\tau-1})$ with A-variables as coefficients.
The evaluation of $\sfe_{rH+\tau}$ has the following form
\begin{align*}
\val(\sfe_{rH+\tau})=\fa_{rH+\tau,rH+\tau}g_{rH+\tau}+\LinMul(a_0, g_1,g_2,\cdots, g_{rH+\tau-1}),
\end{align*}
where $\tilde g_{rH+\tau}\sim \cN(0, I_m)$ is the standard Gaussian vector, 
\begin{align*}
g_{rH+\tau}=\Pi_{Q_\tau^{(\ell)}}^\perp \tilde g_{rH+\tau},
\end{align*}
if the expression $\sfe_{rH+\tau}$ ends with multiplying $W_0^{(\ell)}/\sqrt m$
and 
\begin{align*}
g_{rH+\tau}=\Pi_{P_\tau^{(\ell)}}^\perp \tilde g_{rH+\tau},
\end{align*}
 if the expression $\sfe_{rH+\tau}$ ends with multiplying $(W_0^{(\ell)})^{\top}/\sqrt m$. 
 Moreover, with high probability we have
\begin{align*}
\lim_{m\rightarrow \infty}\fa_{rH+\tau,rH+\tau}=\tilde \fa_{rH+\tau,rH+\tau}\neq 0, \quad \fa_{rH+\tau,rH+\tau}=\tilde \fa_{rH+\tau,rH+\tau}+\OO\left(\frac{(\ln m)^{\fC}}{\sqrt m}\right).
\end{align*}

\end{enumerate}
\end{claim}

\begin{proof}[Proof of Claim \ref{c:newtau}]
We assume that the statements of Claim \ref{c:newtau} hold up to $\tau$ and prove it for $\tau+1$
Without loss of generality, we assume that $\sfe_{rH+\tau+1}$ ends with multiplying $W_0^{(\ell)}/\sqrt m$, then 
$\sfe_{rH+\tau+1}=(W_0^{(\ell)}/\sqrt m)\sff_{rH+\tau+1}$, and $\sff_{rH+\tau+1}\in \Mul(\sfe_0, \sfe_1,\cdots, \sfe_{rH+\tau})$. 
Moreover,
$S^{(\ell)}_{\tau+1}=S^{(\ell)}_{\tau}\cup\{rH+\tau+1\}$, $T^{(\ell)}_{\tau+1}=T^{(\ell)}_\tau
$, and for $2\leq \ell \leq H$,
\begin{align}\begin{split}\label{e:defYQnew}
&Y_{\tau+1}^{(\ell)}=\val\left[\sfe_j\right]_{j\in S^{(\ell)}_{\tau+1}}=[Y_\tau^{(\ell)}, \val(\sfe_{\tau+1})],\quad Q_{\tau+1}^{(\ell)}=\val \left[\frac{\sff_j}{\sqrt m}\right]_{j\in S^{(\ell)}_{\tau+1}}=\left[Q_\tau^{(\ell)}, \frac{\val(\sff_{\tau})}{\sqrt m}\right],\\
&X_{\tau+1}^{(\ell)}=
\val \left[\sfe_j\right]_{j\in T^{(\ell)}_{\tau+1}}=X_\tau^{(\ell)},\quad P_{\tau+1}^{(\ell)}=\val\left[\frac{\sff_j}{\sqrt m}\right]_{j\in T^{(\ell)}_{\tau+1}}=P_\tau^{(\ell)}.
\end{split}\end{align}
By our induction assumption, we have that the limits as $m\rightarrow \infty$ of the Gram matrix of columns of $P_{\tau+1}^{(\ell)}$ is non-degenerate. To prove (${\rm i}$) in Claim \ref{c:newtau}, we only need to show that the limits as $m\rightarrow \infty$ of the Gram matrix of columns of $Q_{\tau+1}^{(\ell)}$ is non-degenerate. We prove it by contradiction. We recall from \eqref{e:defYQnew} $Q_{\tau+1}^{(\ell)}=\val[\sff_j/\sqrt m]_{j\in S_{\tau+1}^{(\ell)}}$. If the limit of the Gram matrix of columns of $Q_{\tau+1}^{(\ell)}$ is degenerate, informally, there exists constants $\la_j$ such that 
\begin{align}\label{e:limit}
\lim_{m\rightarrow \infty} \sum_{j\in S_\tau^{(\ell)}}\la_j \val(\sff_j)+\la_{rH+\tau+1}\val(\sff_{rH+\tau+1})=0.
\end{align} 
We recall that $\sff_j$ is an expression of entrywise products of $\sfe_0$, $\{\sigma^{(s)}(\sfe_i)\}_{0\leq s\leq r-1, 1\leq i\leq rH}$ and $\{\sfe_i\}_{rH+1\leq i\leq j-1}$, and by our induction hypothesis $\sfe_i=\fa_{i,i}g_{i}+\LinMul(a_0, g_1,g_2,\cdots, g_{i-1})$, with $|\fa_{i,i}|\asymp 1$ with high probability.
Moreover, as $m\rightarrow \infty$, the vectors $a_0, g_1, g_2, \cdots, g_{rH+\tau}$ converge to independent standard Gaussian vectors. \eqref{e:limit} implies that as formal expressions
\begin{align*}
\sum_{j\in S_\tau^{(\ell)}}\la_j \sff_j+\la_{rH+\tau+1}\sff_{rH+\tau+1}=0.
\end{align*}
However, this indicates that $\sff_{rH+\tau+1}= \la \sff_j$ with some $j\leq rH+\tau$ and   contradicts with our construction that $\sfe_{rH+\tau+1}$ has not appeared before.
This finishes the proof of (${\rm i}$) in Claim \ref{c:newtau}.

For the proof of (${\rm ii}$) in Claim \ref{c:newtau}, thanks to Lemma \ref{l:decompose}, we have
\begin{align}\label{e:verH}
\val (\sfe_{rH+\tau+1})
&\stackrel{d}{=}
\left(Y_\tau^{(\ell)}(Q_\tau^{(\ell)})^++((P_\tau^{(\ell)})^+)^\top (X_\tau^{(\ell)})^\top\Pi_{Q_\tau^{(\ell)}}^\perp+\Pi_{(P_\tau^{(\ell)})^+}^\perp\tilde W\Pi_{Q_\tau^{(\ell)}}^\perp\right)\frac{\val(\sff_{rH+\tau+1})}{\sqrt m}.
\end{align}
Since $\sff_{r H +\tau+1}\in \Mul(\sfe_0, \sfe_1,\cdots, \sfe_{r H+\tau})$ is an expression of entrywise products of $\sfe_0$, $\{\sigma^{(s)}(\sfe_i)\}_{0\leq s\leq r-1, 1\leq i\leq rH}$ and $\{\sfe_i\}_{rH+1\leq i\leq rH+\tau}$,, and by our induction assumption for $1\leq i\leq rH+\tau$,
\begin{align*}
\val(\sfe_{i})=\fa_{i,i}g_{i}+\LinMul(a_0, g_1,g_2,\cdots, g_{i-1}),
\end{align*}
we conclude that 
\begin{align*}
\val(\sff_{r H +\tau+1})\in\LinMul(a_0, g_1,g_2,\cdots, g_{rH+\tau}).
\end{align*}
By our induction assumption, the columns of $Q_\tau^{(\ell)}$ and $P_\tau^{(\ell)}$ as $m\rightarrow \infty$ are of full rank. The first two terms in \eqref{e:verH} are linear combinations of columns of $Y_\tau^{(\ell)}$ and columns of $P_\tau^{(\ell)}$ with A-variables as coefficients:
\begin{align}\label{e:firsttwo}
\left(Y_\tau^{(\ell)}(Q_\tau^{(\ell)})^++((P_\tau^{(\ell)})^+)^\top (X_\tau^{(\ell)})^\top\Pi_{Q_\tau^{(\ell)}}^\perp\right)\frac{\val(\sff_{rH+\tau+1})}{\sqrt m}\in \LinMul(a, g_1,g_2,\cdots, g_{rH+\tau}).
\end{align}
For the last term in  \eqref{e:verH}, we can rewrite it as
\begin{align}\begin{split}\label{e:verH2}
\Pi_{(P_\tau^{(\ell)})^+}^\perp\tilde W\Pi_{Q_\tau^{(\ell)}}^\perp\val(\sff_{rH+\tau+1})=\fa_{rH+\tau+1,rH+\tau+1}\Pi_{(P_\tau^{(\ell)})^+}^\perp \tilde g_{rH+\tau+1},
\end{split}\end{align}
where $\tilde g_{rH+\tau+1}$ is an independent Gaussian vector and 
\begin{align*}
\fa_{rH+\tau+1,rH+\tau+1}=\frac{1}{\sqrt m}\|\Pi_{Q_\tau^{(\ell)}}^\perp\val(\sff_{rH+\tau+1})\|_2.
\end{align*}
By the same argument as before, the A-variable $\fa_{rH+\tau+1,rH+\tau+1}$ strongly concentrates around this limit. With high probability we have
\begin{align*}
\lim_{m\rightarrow \infty}\fa_{rH+\tau+1,rH+\tau+1}=\tilde \fa_{rH+\tau+1,rH+\tau+1}, \quad \fa_{rH+\tau+1,rH+\tau+1}=\tilde \fa_{rH+\tau+1,rH+\tau+1}+\OO\left(\frac{(\ln m)^{\fC}}{\sqrt m}\right).
\end{align*}
Moreover, as we just proven, (${\rm i}$) in Claim \ref{c:newtau} implies that as $m\rightarrow\infty$, the limit of the Gram matrix of $\{\val(\sff_j)/\sqrt m\}_{j\in S_\tau^{(\ell)}}\cup\{\val(\sff_{rH+\tau+1})/\sqrt m\}$ is non-degenerate. We conclude that $\tilde \fa_{rH+\tau+1 rH+\tau+1}\neq0$, then with high probability $|\fa_{rH+\tau+1,rH+\tau+1}|\asymp 1$. This finishes the proof of Claim \ref{c:newtau}.

\end{proof}

From the discussion above, the evaluation of any subexpression $\sfe_i$ in $K^{(r)}_0(x_{1}, x_{2}, \cdots, x_{r})$
is of the form
\begin{align}\label{e:newform}
\sfe_i=\fa_{i,i}g_i+\LinMul(a_0,g_1,a_2,\cdots, g_{i-1}).
\end{align}
Especially, the vectors $\sfv_j(t)$ at time $t=0$ in \eqref{e:formcc} are also of the form \eqref{e:newform}. Their inner products  concentrate around their limits as $m\rightarrow \infty$,
\begin{align}\label{e:innerp}
 \frac{\langle \sfv_{2j-1}(0),\sfv_{2j}(0)\rangle}{m}
 =\lim_{m\rightarrow\infty}  \frac{\langle \sfv_{2j-1}(0),\sfv_{2j}(0)\rangle}{m}+\OO\left(\frac{(\ln m)^{\fC}}{\sqrt m}\right).
\end{align}
There exists a deterministic operator $\fK^{(r)}: \cX^r\mapsto \bR$ for $2\leq r\leq p+1$, it holds that with high probability
\begin{align*}
\left\|m^{r/2-1}K^{(r)}_0(x_1,x_2,\cdots, x_r)-\fK^{(r)}(x_1,x_2,\cdots,x_r)\right\|_{\infty}\lesssim (\ln m)^\fC.
\end{align*}
By an union bound over all $r$-tuple of data points $(x_{\al_1},x_{\al_2},\cdots, x_{\al_r})$, we conclude that with high probability
\begin{align*}
\left\|K^{(r)}_0-\frac{\fK^{(r)}}{m^{r/2-1}}\right\|_{\infty}\lesssim \frac{(\ln m)^\fC}{m^{(r-1)/2}}.
\end{align*}


If $2\nmid r$, then the degree of $a_0$ in $K_0^{(r)}$ is odd, we have $\bE[K_0^{(r)}]=0$. It is necessary that $\fK^{(r)}=0$. This finishes the proof of Proposition \ref{p:Krlimit}.

\end{proof}

\begin{corollary}\label{c:f0bound}
Under Assumptions \ref{a:sigmaasup} and \ref{a:nonlinear}, for any expression $\sfv(t)\in \fD_r$ with $0\leq r\leq 2p$, the following holds with high probability
%
\begin{align*}
\|\sfv(0)\|_\infty\lesssim (\ln m)^\fC.
\end{align*}
\end{corollary}
\begin{proof}
By the same argument as in the proof of Proposition \ref{p:Krlimit}, we can evaluate $\sfv(0)$ as combinations of standard Gaussian vectors 
\begin{align*}
\sfv(0)\in\LinMul(a_0, g_1,g_2,\cdots),
\end{align*}
where the set $\LinMul$ is as defined in Claim \ref{c:newtau}:
$\LinMul(a_0, g_1,g_2,\cdots)$
is the set of linear combinations of $\Mul(a_0, g_1,g_2,\cdots)$ with A-variables as coefficients;
$\Mul(a_0, g_1,g_2,\cdots)$ is the set of entrywise products of 
$a_0$, $\{\sigma^{(s)}(\fa_{i,i}g_i+\fa_{i,i-1}g_{i-1}+\cdots+\fa_{i,1}g_1)\}_{0\leq s\leq r+1, 1\leq i\leq rH}$ and $\{g_i\}_{i\geq rH+1}$. Since those vectors $g_i$ are projections of independent Gaussian vectors, with high probability $\|g_i\|_\infty\lesssim (\ln m)^{\fC}$. So is any vector in $\LinMul(a_0, g_1,g_2,\cdots)$.

\end{proof}


\section{A Priori Estimates}
\label{s:prior}

In this section, we prove the claim \eqref{e:tprior2} in Theorem \ref{t:main1}.

\begin{proposition}[A priori $L^2$ bounds] \label{p:L2bound}
Under Assumptions \ref{a:sigmaasup} and \ref{a:nonlinear}, for any time $t\geq 0$, we have
\begin{align}\label{e:fL2}
\sum_{\beta=1}^n|f_\beta(t)-y_\beta|^2\leq \sum_{\beta=1}^n|f_\beta(0)-y_\beta|^2=\OO(n),
\end{align}
and with high probability with respect to the random initialization,
for $t\lesssim \sqrt m$ 
\begin{align}\label{e:L2norm}
\frac{1}{\sqrt m}\max\{\|W_t^{(1)}\|_{2\rightarrow 2}, \|W_t^{(2)}\|_{2\rightarrow 2},\|(W_t^{(2)})^\top\|_{2\rightarrow 2}\cdots, \|W_t^{(H)}\|_{2\rightarrow 2}, \|(W_t^{(H)})^\top\|_{2\rightarrow 2},\|a_t\|_{2}\}\lesssim 1.
\end{align}
\end{proposition}

\begin{proof}[Proof of Proposition \ref{p:L2bound}]
From the defining relation \eqref{e:defK2} of $K_t^{(2)}(\cdot,\cdot)$, it is non-negative definite. Using \eqref{e:descent}, we get
\begin{align*}
\del_t \sum_{\beta=1}^n\|f_\beta(t)-y_\beta\|^2=-\sum_{\al, \beta=1}^nK_t^{(2)}(x_\al, x_\beta)(f_\al(t)-y_\al)(f_\beta(t)-y_\beta)\leq 0,
\end{align*}
and \eqref{e:fL2} follows
\begin{align*}
\sum_{\beta=1}^n\|f_\beta(t)-y_\beta\|^2\leq \sum_{\beta=1}^n\|f_\beta(0)-y_\beta\|^2=\OO(n),
\end{align*}

To prove \eqref{e:L2norm}, we define
\begin{align*}
\xi(t)=\frac{1}{\sqrt m}\max\{\|W_t^{(1)}\|_{2\rightarrow 2}, \|W_t^{(2)}\|_{2\rightarrow 2},\|(W_t^{(2)})^\top\|_{2\rightarrow 2}\cdots, \|W_t^{(H)}\|_{2\rightarrow 2}, \|(W_t^{(H)})^\top\|_{2\rightarrow 2},\|a_t\|_{2}\}.
\end{align*}
We notice that for $t=0$, $W_0^{(1)}$ is an $m\times d$ random gaussian matrix, $W_0^{(2)}, W_0^{(3)},\cdots, W_0^{(H)}$ are $m\times m$ random gaussian matrices, and $a_0$ is a gaussian vector of length $m$. { From random matrix theory \cite{MR2963170}}, we have that,  with high probability,  
\begin{align}\label{e:initbound}
\xi(0)\lesssim 1.
\end{align}

In the following we derive an upper bound of $\del_t\xi(t)$, which combining with \eqref{e:initbound} gives us the desired bound \eqref{e:L2norm}. For any $\ell\geq 1 $, we have 
\begin{align*}\begin{split}
\|x^{(\ell)}\|_2
&=\frac{1}{\sqrt m}\|\sigma(W^{(\ell)} x^{(\ell-1)})\|_2
\leq  \frac{1}{\sqrt m} \sqrt{\sum_{i=1}^m(|\sigma(0)|+\fc_1 (W^{(\ell)} x^{(\ell-1)})_i)^2}\\
&\leq |\sigma(0)|+\frac{\fc_1}{\sqrt m}\|W^{(\ell)} x^{(\ell-1)}\|_2
\leq |\sigma(0)|+\fc_1\xi(t)\|x^{(\ell-1)}\|_2,
\end{split}\end{align*}
where we used Assumption \ref{a:sigmaasup} that $\sigma$ is $\fc_1$-Lipschitz, and Assumption \ref{a:nonlinear} $\|x\|_2\lesssim 1$. Inductively, we have the following estimate
\begin{align}\label{e:xlbound}
\|x^{(\ell)}\|_2\leq \fc_1^{\ell} \xi(t)^{\ell}\|x\|_2+|\sigma(0)|(1+\fc_1\xi(t)+\cdots+\fc_1^{\ell-1} \xi(t)^{\ell-1})\lesssim \fc_1^\ell \xi(t)^\ell.
\end{align}

Using \eqref{e:derW},  \eqref{e:dera}  and \eqref{e:xlbound}, we have the following  bounds: 
\begin{align}\begin{split}\label{e:derWbound}
&\phantom{{}={}}\del_t \|W^{(\ell)}_t\|_{2\rightarrow 2}
\leq \frac{1}{n}\sum_{\beta=1}^n\left\|\sigma'_\ell(x_\beta)\frac{(W_t^{(\ell+1)})^\top}{\sqrt m}\cdots \sigma'_{H}(x_\beta)\frac{a_t}{\sqrt m}\right\|_2 \|x_\beta^{(\ell-1)}\|_2|f_\beta(t)-y_\beta|\\
&
\lesssim \frac{1}{n}\sum_{\beta=1}^n \fc_1^{H} \xi(t)^{H} |f_\beta(t)-y_\beta|
\lesssim \fc_1^{H} \xi(t)^{H}\sqrt{\frac{1}{n}\sum_{\beta=1}^n  |f_\beta(t)-y_\beta|^2}
\lesssim \fc_1^{H} \xi(t)^{H},
\end{split}\end{align}
and
\begin{align}\label{e:derabound}\begin{split}
&\phantom{{}={}}\del_t \|a(t)\|_2\leq \frac{1}{n}\sum_{\beta =1}^n \|x_\beta^{(H)}\|_2|f_\beta(t)-y_\beta|
\lesssim \frac{1}{n}\sum_{\beta=1}^n \fc_1^{H} \xi(t)^{H}(1+\|x_\beta\|_2) |f_\beta(t)-y_\beta|\\
&\lesssim \fc_1^{H} \xi(t)^{H}\sqrt{\frac{1}{n}\sum_{\beta=1}^n  |f_\beta(t)-y_\beta|^2}
\lesssim \fc_1^{H} \xi(t)^{H},
\end{split}\end{align}
where we used the  AM-GM inequality and \eqref{e:fL2}. And similarly, 
\begin{align}\label{e:derWTbound}
\del_t \|(W^{(\ell)}_t)^\top\|_{2\rightarrow 2}\lesssim \fc_1^{H} \xi(t)^{H}.
\end{align}

The estimates \eqref{e:derWbound}, \eqref{e:derabound} and \eqref{e:derWTbound} together implies the upper bounds for $\del_t \xi(t)$: there exists some large constant $\fC>0$
\begin{align*}
\del_t\xi(t)\leq \frac{\fC \fc_1^{H}}{\sqrt m} \xi(t)^H.
\end{align*}
If $H=1$, we have 
\begin{align*}
\xi(t)\leq e^{\fC \fc_1^H t/\sqrt m}\xi(0),
\end{align*}
and for $H\geq 2$, we get
\begin{align*}
\xi(t)\leq \left(\xi(0)^{H-1}-\fC\fc_1^Ht/\sqrt m\right)^{-1/(H-1)}.
\end{align*}
In both cases, we have that $\xi(t)\lesssim 1$ provided that $t\lesssim \sqrt m$, where the implicit constants depend on the depth $H$. This finishes the proof of Proposition \ref{p:L2bound}.
\end{proof}

As we have shown in Section \ref{s:outline} \eqref{e:form}, the kernel $K^{(r)}_t(x_{\al_1}, x_{\al_2}, \cdots, x_{\al_r})$ is a sum of terms in the form
\begin{align}\label{e:formcopy}
\frac{1}{m^{r/2-1}}\prod_{j=1}^s \frac{\langle \sfv_{2j-1}(t),\sfv_{2j}(t)\rangle}{m}, \quad 1\leq s\leq r,\quad \sfv_j(t)\in \fD_0\cup \fD_1\cup \cdots \cup \fD_{r-2}.
\end{align}
In the following we derive an upper bound of $\|\diag(\sff_t)\|_{2\rightarrow 2}=\|\sff_t\|_\infty$, where $\diag(\cdots)$ is the diagonalization of a vector, for any $\sff_t\in \fD_0\cup\fD_1\cup\cdots\cup \fD_r$. 

\begin{proposition}\label{p:upperb}
We assume Assumptions \ref{a:sigmaasup} and \ref{a:nonlinear}.
Fix time $t\geq 0$ and $r\geq 0$. Suppose that  for all expressions $\sff_t\in \fD_0\cup\fD_1\cup
\cdots\cup \fD_r$, the following holds
\begin{align}\label{e:smallb}
\|\diag(\sff_t)\|_{2\rightarrow 2}\leq M, 
\end{align}
for some constant $M \ge 1$. \nc 
Then for any $\sff_t\in \fD_s$ with $0\leq s\leq 2r+2$, the following holds
\begin{align}\label{e:bigb}
\|\diag(\sff_t)\|_{2\rightarrow 2}\leq \|\sff_t\|_{2}  \lesssim M^{s}\sqrt m.
\end{align}
\end{proposition}
\begin{proof}[Proof of Proposition \ref{p:upperb}]
%

We notice that $\|\diag(\sff_t)\|_{2\rightarrow 2}$ equals the $L_\infty$ norm of the vector $\sff_t$. The $L_\infty$ norm of $\sff_t$ is bounded by its $L_2$ norm, 
\begin{align*}
\|\diag(\sff_t)\|_{2\rightarrow 2}=\|\sff_t\|_\infty\leq \|\sff_t\|_2,
\end{align*}
which gives the first inequality of \eqref{e:bigb}. Notice that the last  inequality above 
is not optimal and it costs a factor $\sqrt m$  generically. 

For $\sff_t\in \fD_s$ with some $s\leq 2r+2$. We can write it as $\sff_t=\sfe_k \sfe_{k-1} \cdots \sfe_1\sfe_0$ where $0\leq k\leq 4H-3$,
\begin{align*}
\sfe_0\in \left\{a_t, \bm1,\{\sqrt m x^{(1)}_\beta, \sqrt m x^{(2)}_\beta, \cdots, \sqrt m x^{(H)}_\beta\}_{1\leq \beta\leq n}\right\},
\end{align*}
and for $1\leq j\leq k$, $\sfe_j$ belongs to one of the sets
\begin{align}
\label{e:Wt}& \left\{\left\{\frac{W_t^{(2)}}{\sqrt m},\frac{(W_t^{(2)})^\top}{\sqrt m},\cdots, \frac{W_t^{(H)}}{\sqrt m},\frac{(W_t^{(H)})^\top}{\sqrt m}\right\} , \{\sigma_1'(x_\beta), \sigma_2'(x_\beta), \cdots, \sigma_H'(x_\beta)\}_{1\leq \beta\leq n}\right\},\\
\label{e:sig}&\left\{\diag(\sfd),\quad \sfd\in \fD_0\cup\fD_1\cup\cdots\cup \fD_{s-1}\right\},\\
\begin{split}
\label{e:diagd}&\left\{\sigma_\ell^{(u+1)}(x_\beta)\diag(\sfd_1)\diag(\sfd_2)\cdots \diag(\sfd_u): 1\leq \ell\leq H, \right. \\
&\left.\phantom{\sfe_j\in \left\{\sigma_\ell^{(u+1)}\right\}}1\leq \beta\leq n, 1\leq u\leq s-1, \sfd_1,\sfd_2,\cdots,\sfd_u\in \fD_0\cup \fD_1\cup\cdots\cup \fD_{s-1}\right\}.
\end{split}
\end{align}
Moreover, 
the total number of $\diag$ operations in the expression $\sff_t=\sfe_k\sfe_{k-1}\cdots \sfe_1\sfe_0$ is exactly $s$. We remark that  $x^{(\ell)}_\beta$ depends on time $t$.
%

In the following we prove by induction on $s$ that 
\begin{align}\label{e:eL2bound}
\|\sff_t\|_{2}=\|\sfe_k\sfe_{k-1}\cdots\sfe_1\sfe_0\|_2\lesssim M^{s} \sqrt m,
\end{align}
which gives the claim \eqref{e:bigb}.

For $s=0$, $\sff_t$ does not contain $\diag$ operations, and all the $\sfe_j$ belong to \eqref{e:Wt}. In this case, thanks to Assumption \ref{a:sigmaasup} and Proposition \ref{p:L2bound}, with high probability with respect to the random 
initialization,  $\|\sfe_0\|_2\lesssim \sqrt m$ and $\|\sfe_{j}\|_{2\rightarrow 2}\lesssim 1$ for all $1\leq j\leq k$. Therefore, we have that 
\begin{align*}
\|\sff_t\|_{2} =\|\sfe_k\sfe_{k-1}\cdots\sfe_1\sfe_0\|_2
\leq \|\sfe_k\|_{2\rightarrow 2}\|\sfe_{k-1}\|_{2\rightarrow 2}\cdots\|\sfe_1\|_{2\rightarrow 2}\|\sfe_0\|_2
\lesssim \sqrt m.
\end{align*}
For $1\leq s\leq r$, by our assumption \eqref{e:smallb}
\begin{align*}
 \|\sff_t\|_2=\|\sfe_k\sfe_{k-1}\cdots\sfe_1\sfe_0\|_2\leq \sqrt m\|\sfe_k\sfe_{k-1}\cdots\sfe_1\sfe_0\|_\infty=\sqrt m \|\diag(\sff_t)\|_{2\rightarrow 2}\leq M\sqrt m\leq M^s \sqrt m.
\end{align*}
Thus the claim \eqref{e:eL2bound} holds for any $s\leq r$.

In the following we assume that \eqref{e:eL2bound} holds for $1,2,\cdots, s-1$ and prove it for $s$. 
For each $1\leq j\leq k$, we denote the number of $\diag$ operations in $\sfe_j$ by  $s_j$.  Then the total number of $\diag$ operations in $\sff_t$ is $s_1+s_2+\cdots+s_k=s\leq 2r+2$. As an easy consequence, $s_j\leq s\leq 2r+2$ for any $1\leq j\leq k$. For each term $\sfe_j$, there are several cases:
\begin{enumerate}[(i)]
\item  $\sfe_j$ belongs to \eqref{e:Wt}, then it does not contain any $\diag$ operation, and we have $s_j=0$. In this case, we have proven that $\|\sfe_j\|_{2\rightarrow 2}\lesssim 1$.
\item $\sfe_j=\diag(\sfd)$ belongs to \eqref{e:sig} and $s_j\leq r+1$. Since the expression $\sfd$ contains $s_j-1\leq r$ 
$\diag$ operations,  $\sfd$ also belongs to $ \fD_{s_j-1}\subset \fD_{0}\cup\fD_{1}\cup \cdots \cup\fD_{r}$. 
   By the assumption \eqref{e:smallb}
\begin{align*}
\|\sfe_j\|_{2\rightarrow 2}=\|\diag(\sfd)\|_{2\rightarrow 2}\leq M.
\end{align*}
\item $\sfe_j=\diag(\sfd)$ belongs to \eqref{e:sig} and $s_j>r+1$. In this case we still have that $s_j-1\leq s-1$. And by our induction assumption \eqref{e:eL2bound} and $\sfd\in \fD_{s_j-1}$, it holds
\begin{align}\label{e:estd}
   \|\sfd\|_{2} 
   \lesssim  M^{s_j-1}\sqrt m.
\end{align}
\item  $\sfe_j=\sigma_\ell^{(u+1)}(x_\beta)\diag(\sfd_1)\diag(\sfd_2)\cdots \diag(\sfd_u)$ belongs to \eqref{e:diagd}, and each of those subexpressions $\diag(\sfd_1), \diag(\sfd_2), \cdots, \diag(\sfd_u)$ contains at most $r+1$ $\diag$ operations. In this case we have $u\leq s_j$ and $\sfd_1, \sfd_2,\cdots,\sfd_u\in \fD_{0}\cup\fD_{1}\cup \cdots \cup\fD_{r}$. By the assumption \eqref{e:smallb} and Assumption \ref{a:sigmaasup}
\begin{align*}\begin{split}
\|\sfe_j\|_{2\rightarrow 2}
&=\|\sigma_\ell^{(u+1)}(x_\beta)\|_{2\rightarrow 2}\|\diag(\sfd_1)\|_{2\rightarrow 2}\|\diag(\sfd_2)\|_{2\rightarrow 2}\cdots \|\diag(\sfd_u)\|_{2\rightarrow 2}\\
&\leq \fc_{u+1}M^u\lesssim M^{s_j}
\end{split}\end{align*}

\item  $\sfe_j=\sigma_\ell^{(u+1)}(x_\beta)\diag(\sfd_1)\diag(\sfd_2)\cdots \diag(\sfd_u)$ belongs to \eqref{e:diagd},  and some of those subexpressions $\diag(\sfd_1), \diag(\sfd_2), \cdots, \diag(\sfd_u)$ contain more than $r+1$ $\diag$ operations. In this case $s_j> r+1$. Since the total number of $\diag$ operations in $\sfe_j$ is $s_j\leq 2r+2$, exact one of $\diag(\sfd_1), \diag(\sfd_2), \cdots, \diag(\sfd_u)$ contains more than $r+1$ $\diag$ operations. Say it is $\diag(\sfd_v)$. For any $1\leq i\neq v\leq u$, $\diag(\sfd_i)$ contains at most $r+1$ $\diag$ operations and $\sfd_i\in \fD_{0}\cup\fD_{1}\cup \cdots \cup\fD_{r}$. By the assumption \eqref{e:smallb} 
\begin{align}\begin{split}\label{e:oterm}
\|\sfd_i\|_\infty=\|\diag(\sfd_i)\|_{2\rightarrow 2}\leq M.
\end{split}\end{align}
For $\sfd_v$, it contains at most  $s_j-u\leq s-1$ $\diag$ operations. Thus by our induction assumption \eqref{e:eL2bound}, it holds
\begin{align}\label{e:vterm}
\|\sfd_v\|_{2}\lesssim  M^{s_j-u}\sqrt m.
\end{align}

\end{enumerate}

By our assumption, the total number of $\diag$ operations in $\sff_t=\sfe_k\sfe_{k-1}\cdots\sfe_0$ is  $s_1+s_2+\cdots+s_k=s\leq 2r+2$. At most one of those $s_j$ is bigger than $r+1$. Especially at most one of those $\sfe_j$ belongs to cases (iii) or (v). 

If none of those $\sfe_j$ belongs to cases (iii) or (v), then with the bound $\|\sfe_0\|_2\lesssim \sqrt m$, we have 
\begin{align}\label{e:cc1}
\|\sfe_k\sfe_{k-1}\cdots\sfe_1\sfe_0\|_2
\leq \|\sfe_k\|_{2\rightarrow 2}\|\sfe_{k-1}\|_{2\rightarrow 2}\cdots\|\sfe_1\|_{2\rightarrow 2}\|\sfe_0\|_2
\lesssim \prod_{i=1}^k M^{s_i} \sqrt m=M^{s}\sqrt m.
\end{align}

If for some $1\leq j\leq k$, $\sfe_j$ belongs to the case (iii), we write
\begin{align}\begin{split}\label{e:b1}
\|\sfe_k\sfe_{k-1}\cdots\sfe_1\sfe_0\|_2
&\leq  \|\sfe_k\|_{2\rightarrow 2}\|\sfe_{k-1}\|_{2\rightarrow 2}\cdots\|\sfe_{j+1}\|_{2\rightarrow 2}
\|\sfe_j\sfe_{j-1}\cdots \sfe_1\sfe_0\|_{2}\\
&\leq  \|\sfe_k\|_{2\rightarrow 2}\|\sfe_{k-1}\|_{2\rightarrow 2}\cdots\|\sfe_{j+1}\|_{2\rightarrow 2}
\|\sfd\|_2\|\sfe_{j-1}\sfe_{j-2}\cdots \sfe_1\sfe_0\|_{\infty}.
\end{split}\end{align}
The expression $\sfe_{j-1}\sfe_{j-2}\cdots \sfe_1\sfe_0$ contains at most $s_1+s_2+\cdots+s_{j-1}\leq r$ $\diag$ operators, and $j-1\leq k-1<4H-3$. Therefore, the expression $\sfe_{j-1}\sfe_{j-2}\cdots \sfe_1\sfe_0$ is in the set $\fD_{s_1+s_2+\cdots+s_{j-1}}$, which is contained in $\fD_0\cup\fD_1\cup\cdots\cup\fD_{r}$. Thus by our assumption \eqref{e:smallb}, 
\begin{align}\label{e:b2}
\|\sfe_{j-1}\sfe_{j-2}\cdots \sfe_1\sfe_0\|_\infty=\|\diag(\sfe_{j-1}\sfe_{j-2}\cdots \sfe_1\sfe_0)\|_{2\rightarrow2}
\leq M.
\end{align}
We estimate $\|\sfd\|_2$ using \eqref{e:estd}, and estimate $\|\sfe_k\|_{2\rightarrow 2} \|\sfe_{k-1}\|_{2\rightarrow 2}\cdots\|\sfe_{j+1}\|_{2\rightarrow 2}$ by (i), (ii) and (iv).
By plugging \eqref{e:estd}, \eqref{e:b2} into \eqref{e:b1}, we get
\begin{align*}\begin{split}
&\phantom{{}={}}\|\sfe_k\|_{2\rightarrow 2}\|\sfe_{k-1}\|_{2\rightarrow 2}\cdots\|\sfe_{j+1}\|_{2\rightarrow 2}
\|\sfd\|_2\|\sfe_{j-1}\sfe_{j-2}\cdots \sfe_1\sfe_0\|_{\infty}\\
&\lesssim \prod_{i=j+1}^k M^{s_i} (M^{s_j-1}\sqrt m) M
\lesssim M^{s}\sqrt m.
\end{split}\end{align*}

If for some $1\leq j\leq k$, $\sfe_j$ belongs to the case (iii), we write
\begin{align}\begin{split}\label{e:bb1}
&\phantom{{}={}}\|\sfe_k\sfe_{k-1}\cdots\sfe_1\sfe_0\|_2
\leq  \|\sfe_k\|_{2\rightarrow 2}\|\sfe_{k-1}\|_{2\rightarrow 2}\cdots\|\sfe_{j+1}\|_{2\rightarrow 2}
\|\sfe_j\sfe_{j-1}\cdots \sfe_1\sfe_0\|_{2}\\
&= \|\sfe_k\|_{2\rightarrow 2}\|\sfe_{k-1}\|_{2\rightarrow 2}\cdots\|\sfe_{j+1}\|_{2\rightarrow 2}
\|\sigma_\ell^{(u+1)}(x_\beta)\diag(\sfd_1)\diag(\sfd_2)\cdots \diag(\sfd_u)\sfe_{j-1}\cdots \sfe_1\sfe_0\|_{2}.
\end{split}\end{align}
For the last term in \eqref{e:bb1}, we have
\begin{align}\begin{split}\label{e:bb2}
&\phantom{{}={}}\|\sigma_\ell^{(u+1)}(x_\beta)\diag(\sfd_1)\diag(\sfd_2)\cdots \diag(\sfd_u)\sfe_{j-1}\cdots \sfe_1\sfe_0\|_{2}\\
&\leq
\|\sigma_\ell^{(u+1)}(x_\beta)\diag(\sfd_1)\cdots \diag(\sfd_{v-1})\|_{2\rightarrow 2}\|\sfd_v\|_2\|\diag(\sfd_{v+1})\cdots \diag(\sfd_{u})\sfe_{j-1}\cdots \sfe_1\sfe_0\|_{\infty}\\
&\leq
\|\sigma_\ell^{(u+1)}(x_\beta)\|_{2\rightarrow 2}\|\sfd_1\|_\infty\cdots \|\sfd_{v-1}\|_\infty\|\sfd_v\|_2\|\sfd_{v+1}\|_\infty\cdots \|\sfd_{u}\|_{\infty}\|\sfe_{j-1}\cdots \sfe_1\sfe_0\|_{\infty}.
\end{split}\end{align}
Plugging \eqref{e:oterm}, \eqref{e:vterm} and \eqref{e:b2} into \eqref{e:bb2}, we get
\begin{align}\begin{split}\label{e:bb3}
&\phantom{{}={}}\|\sigma_\ell^{(u+1)}(x_\beta)\diag(\sfd_1)\diag(\sfd_2)\cdots \diag(\sfd_u)\sfe_{j-1}\cdots \sfe_1\sfe_0\|_{2}\\
&\lesssim M^{u-1} \|\sfd_v\|_2\| \sfe_{j-1}\cdots \sfe_1\sfe_0\|_{2}
\lesssim M^{u-1} M^{s_j-u}\sqrt m M= M^{s_j}\sqrt m.
\end{split}\end{align}
We estimate $\|\sfe_k\|_{2\rightarrow 2} \|\sfe_{k-1}\|_{2\rightarrow 2}\cdots\|\sfe_{j+1}\|_{2\rightarrow 2}$ by (i), (ii) and (iv).
By plugging \eqref{e:bb3} into \eqref{e:bb1}, we get
\begin{align*}\begin{split}
&\phantom{{}={}}\|\sfe_k\|_{2\rightarrow 2}\|\sfe_{k-1}\|_{2\rightarrow 2}\cdots\|\sfe_{j+1}\|_{2\rightarrow 2}
\|\sfe_j\sfe_{j-1}\sfe_{j-2}\cdots \sfe_1\sfe_0\|_{\infty}\\
&\lesssim \prod_{i=j+1}^k M^{s_i} (M^{s_j}\sqrt m) 
\lesssim M^{s}\sqrt m.
\end{split}\end{align*}
This finishes the proof of \eqref{e:eL2bound}, and hence Proposition \ref{p:upperb}. 
Notice that the inequalities \eqref{e:estd} and \eqref{e:vterm} which contain the factor $\sqrt m$ were used only once in this proof.
\end{proof}

\begin{proposition}\label{p:ftbound}
We assume Assumptions \ref{a:sigmaasup} and \ref{a:nonlinear}. With high probability, uniformly for any vector $\sff_t\in \fD_0\cup\fD_1\cup
\cdots\cup \fD_{p-1}$, and time $0\leq t\leq m^{\frac{p}{2(p+1)}}/(\ln m)^{\fC'}$ the following holds
\begin{align*}
\|\sff_t\|_\infty\lesssim (\ln m)^\fC.
\end{align*}
\end{proposition}

\begin{proof}[Proof of Proposition \ref{p:ftbound}]
Thanks to Corollary \ref{c:f0bound}, with high probability, uniformly for all $\sff_t\in \fD_0\cup\fD_1\cup\cdots\cup \fD_{p-1}$, we have that 
\begin{align*}
\|\sff_0\|_{\infty}\lesssim (\ln m)^{\fC}.
\end{align*}
We denote 
\begin{align}\label{e:defxi2}
\xi(t)=\max\{\|\sff_t\|_\infty: \sff_t\in \fD_0\cup\fD_1\cup\cdots\cup \fD_{p-1}\}.
\end{align}
In the following we derive a self-consistent differential equation of $\xi(t)$. Proposition \ref{p:ftbound} follows from analyzing it.

For any $\sff_t(x_{\al_1}, x_{\al_2},\cdots, x_{\al_s})\in \fD_0\cup\fD_1\cup\cdots\cup \fD_{p-1}$, by taking derivative we have
\begin{align}\label{e:der1}
\del_t\sff_t(x_{\al_1}, x_{\al_2}, \cdots, x_{\al_s})
=-\frac{1}{n}\sum_{\beta=1}^n\frac{1}{\sqrt m}\sff_t^{(1)}(x_{\al_1}, x_{\al_2}, \cdots, x_{\al_s}, x_\beta)(f_\beta(t)-y_\beta).
\end{align}
where $\sff_t^{(1)}(x_{\al_1}, x_{\al_2}, \cdots, x_{\al_s}, x_\beta)$ is obtained from $\sff_t(x_{\al_1}, x_{\al_2}, \cdots, x_{\al_s})$ by the replacements \eqref{e:replace}. 
We define,
\begin{align*}
 \LinProd(\fD_0\cup\fD_1\cup\cdots\cup \fD_{p+r})
=
\left\{\text{linear combinations of }\sfv_0(t)\prod_{j=1}^u \frac{\langle\sfv_{2j-1}(t),\sfv_{2j}(t)\rangle}{m}\right\},
\end{align*}
where $0\leq u\leq r+1$, $\sfv_j(t)\in \fD_{s_j}$,  $s_0, s_1, \cdots, s_{2u}\geq 0$ and $s_0+s_1+\cdots+s_{2u}\leq p+r$.
Thanks to Proposition \ref{p:fDstruc}, $\sff_t^{(1)}\in \LinProd(\fD_0\cup\fD_1\cup\cdots\cup \fD_p)$. More generally, for any integer $1\leq r\leq p+1$,
%
%
\begin{align}\label{e:der2}
\del_t\sff_t^{(r)}(x_{\al_1}, x_{\al_2}, \cdots,x_{\al_{s+r}})
=-\frac{1}{n}\sum_{\beta=1}^n\frac{1}{\sqrt m}\sff_t^{(r+1)}(x_{\al_1}, x_{\al_2}, \cdots, x_{\al_{s+r}}, x_\beta)(f_\beta(t)-y_\beta),
\end{align}
where $\sff_t^{(r+1)}\in \LinProd(\fD_0\cup\fD_1\cup\cdots\cup \fD_{p+r})$.
Using the bound \eqref{e:fL2}, we have
\begin{align*}\begin{split}
&\phantom{{}={}}\left\|\frac{1}{n}\sum_{\beta=1}^n\frac{1}{\sqrt m}\sff_t^{(r+1)}(x_{\al_1}, x_{\al_2}, \cdots, x_{\al_{s+r}}, x_\beta)(f_\beta(t)-y_\beta)\right\|_\infty\\
&\leq 
\frac{1}{\sqrt m}\max_{1\leq \beta \leq n}\|\sff_t^{(r+1)}(x_{\al_1}, x_{\al_2}, \cdots, x_{\al_{s+r}}, x_\beta)\|_\infty\frac{1}{n}\sum_{\beta=1}^n|f_\beta(t)-y_\beta|\\
&\leq \frac{1}{\sqrt m}\max_{1\leq \beta \leq n}\|\sff_t^{(r+1)}(x_{\al_1}, x_{\al_2}, \cdots, x_{\al_{s+r}}, x_\beta)\|_\infty\sqrt{\frac{1}{n}\sum_{\beta=1}^n(f_\beta(t)-y_\beta)^2}\\
&\lesssim \frac{1}{\sqrt m}\max_{1\leq \beta \leq n}\|\sff_t^{(r+1)}(x_{\al_1}, x_{\al_2}, \cdots, x_{\al_{s+r}}, x_\beta)\|_\infty.
\end{split}\end{align*}
Therefore, \eqref{e:der1} and \eqref{e:der2} together give 
\begin{align}
\label{e:der1b}&\|\del_t\sff_t(x_{\al_1}, x_{\al_2}, \cdots,x_{\al_{s}})\|_\infty
\lesssim \frac{1}{\sqrt m}\max_{1\leq \beta \leq n}\|\sff_t^{(1)}(x_{\al_1}, x_{\al_2}, \cdots, x_{\al_{s}}, x_\beta)\|_\infty,\\
\label{e:der2b}&\|\del_t\sff_t^{(r)}(x_{\al_1}, x_{\al_2}, \cdots,x_{\al_{s+r}})\|_\infty
\lesssim \frac{1}{\sqrt m}\max_{1\leq \beta \leq n}\|\sff_t^{(r+1)}(x_{\al_1}, x_{\al_2}, \cdots, x_{\al_{s+r}}, x_\beta)\|_\infty,
\end{align}
for any $1\leq r\leq p+1$. By taking higher derivatives on both sides of \eqref{e:der1b}, and using \eqref{e:der2b} to bound the righthand side, we have that
\begin{align}\begin{split}\label{e:ftder}
&\phantom{{}={}}\del_t^{(p+1)}\|\sff_t(x_{\al_1}, x_{\al_2}, \cdots,x_{\al_{s}})\|_\infty\\
&\lesssim \frac{1}{\sqrt m}\del_t^{(p)}\max_{1\leq \beta_1\leq n}\|\sff_t^{(1)}(x_{\al_1}, x_{\al_2}, \cdots,x_{\al_{s}}, x_{\beta_1})\|_\infty\lesssim \cdots\cdots\\
&\lesssim \frac{1}{m^{(p+1)/2}}\max_{1\leq \beta_1,\beta_2,\cdots,\beta_{p+1} \leq n}\|\sff_t^{(p+1)}(x_{\al_1}, x_{\al_2}, \cdots, x_{\al_{s}}, x_{\beta_1}, x_{\beta_2},\cdots, x_{\beta_{p+1}})\|_\infty.
\end{split}\end{align}

From the discussion above, $\sff_t^{(p+1)}$ is a linear combination of terms in the form
\begin{align}\label{e:form2}
\sfv_0(t)\prod_{j=1}^u \frac{\langle\sfv_{2j-1}(t),\sfv_{2j}(t)\rangle}{m},
\end{align}
where $0\leq u\leq p+1$, $\sfv_j(t)\in \fD_{s_j}$, $s_0, s_1,\cdots, s_{2u}\geq 0$ and $s_0+s_1+\cdots+s_{2u}\leq 2p$. We can use Proposition \ref{p:upperb} for $r=p-1$,
\begin{align}
\|\sfv_0(t)\|_\infty\lesssim \xi(t)^{s_0}\sqrt m,\label{e:vLinfty}
\end{align}
and for $1\leq j\leq 2u$
\begin{align}
\|\sfv_j(t)\|_2\lesssim \xi(t)^{s_j}\sqrt m.\label{e:vL2}
\end{align}
The estimates \eqref{e:vLinfty} and \eqref{e:vL2} together give an upper bound for the $L_\infty$ norm of $\sff_t^{(p+1)}$,
\begin{align}\label{e:ftbound}
\|\sff_t^{(p+1)}(x_{\al_1}, x_{\al_2}, \cdots, x_{\al_{s}}, x_{\beta_1}, x_{\beta_2},\cdots, x_{\beta_{p+1}})\|_\infty\lesssim 
\sqrt m \prod_{j=0}^{u}\xi(t)^{s_j}\lesssim  \xi(t)^{2p} \sqrt m.
\end{align}
We obtain a self-consistent differential equation of $\xi(t)$ by taking maximum on both sides of \eqref{e:ftder} over $1\leq \al_1,\al_2,\cdots,\al_s\leq n$, and using \eqref{e:ftbound}
\begin{align}\label{e:xieq}
\del^{(p+1)}_{t}\xi(t)\lesssim \frac{\xi(t)^{2p}}{m^{p/2}}.
\end{align}

To obtain an upper bound of $\xi(t)$ using \eqref{e:xieq}, we still need an upper bound for the initial data, i.e. $\xi(0)$ and $\{\del_t^{(r)}\xi(0)\}_{1\leq r\leq p}$. Fortunately Corollary \ref{c:f0bound} provides such estimates. In fact, Corollary \ref{c:f0bound} implies that with high probability $\xi(0)\lesssim (\ln m)^{\fC}$. For the derivatives of $\xi(t)$ at $t=0$, we use \eqref{e:ftder}
\begin{align*}\begin{split}
&\phantom{{}={}}|\del_t^{(r)} \xi(0)|
\lesssim 
\left.\max_{1\leq\al_1,\al_2,\cdots,\al_s\leq n}\del_t^{(r)}\|\sff_t(x_{\al_1}, x_{\al_2}, \cdots,x_{\al_{s}})\|_\infty \right|_{t=0} \\
&\lesssim 
\frac{1}{m^{r/2}}\max_{1\leq\al_1,\al_2,\cdots,\al_s\leq n}\max_{1\leq \beta_1,\beta_2,\cdots,\beta_{p+1} \leq n}\|\sff_0^{(r)}(x_{\al_1}, x_{\al_2}, \cdots, x_{\al_{s}}, x_{\beta_1}, x_{\beta_2},\cdots, x_{\beta_{r}})\|_\infty.
\end{split}\end{align*}
Again $\sff_0^{(r)}$ is a linear combination of terms in the form \eqref{e:form2} with $\sfv_j(t)\in \fD_{s_j}$ for some $0\leq u\leq r$, $s_0, s_1,\cdots, s_{2u}\geq 0$ and $s_0+s_1+\cdots+s_{2u}\leq p+r-1$. Using Corollary \ref{c:f0bound}, for $0\leq j\leq 2u$,  $\|\sfv_j(0)\|_\infty\lesssim (\ln m)^{\fC}$. We conclude that
\begin{align*}
|\del_t^{(r)} \xi(0)|
\lesssim \frac{(\ln m)^{(2r+1)\fC}}{m^{r/2}},
\end{align*}
for any $1\leq r\leq p$.

The ordinary differential equation \eqref{e:xieq} has an exact solution in the following form:
\begin{align*}
\tilde \xi(t)=\frac{A_0 m^{\frac{p}{2(2p-1)}}}{\left(A_1m^{\frac{p}{2(p+1)}}/(\ln m)^{\frac{2p-1}{p+1}\fC}-t\right)^{\frac{p+1}{2p-1}}},
\end{align*}
where $A_0, A_1$ are constants depending on $p$, which are chosen such that $\tilde \xi(t)$ is an exact solution of \eqref{e:xieq}, and $\tilde \xi(0)=\xi(0)$. It is easy to check that $\tilde \xi(0)\asymp (\ln m)^{\fC}$, and for $1\leq r\leq p+1$,
\begin{align*}
\del_t^{(r)}\tilde \xi(0)\asymp (\ln m)^{\left(1+\frac{(2p-1)r}{p+1}\right)\fC}m^{-\frac{pr}{2(p+1)}}\gg\frac{(\ln m)^{(2r+1)\fC}}{m^{r/2}}\asymp \del_t^{(r)}\xi(0),
\end{align*}
provided that $m$ is large enough.
Therefore, $\tilde \xi(t)$ provides an upper bound for $\xi(t)$. We conclude that for 
\begin{align*}
t\lesssim m^{\frac{p}{2(p+1)}}/(\ln m)^{\frac{2p-1}{p+1}\fC},
\end{align*}
it holds that
\begin{align*}
\xi(t)\lesssim \tilde \xi(t)\lesssim 
\frac{A_0 m^{\frac{p}{2(2p-1)}}}{\left(A_1m^{\frac{p}{2(p+1)}}/(\ln m)^{\frac{2p-1}{p+1}\fC}\right)^{\frac{p+1}{2p-1}}}\lesssim (\ln m)^{\fC}.
\end{align*}
This finishes the proof of Proposition \ref{p:ftbound}.
\end{proof}

\begin{remark}
By the same argument as for \eqref{e:xieq}, for any $0 \leq r\leq p$, we have
\begin{align*}
\del^{(r+1)}_{t}\xi(t)\lesssim \frac{\xi(t)^{p+r}}{m^{r/2}},
\end{align*}
which gives us that $\xi(t)\lesssim (\ln m)^\fC$ for $t\leq m^{\frac{r}{2(r+1)}}/(\ln m)^{\fC'}$. Therefore, for bigger $r$, we have the a prior estimate $\xi(t)\lesssim (\ln m)^\fC$ for longer time. 

\end{remark}

\begin{proof}[Proof of \eqref{e:tprior2} in Theorem \ref{t:main1}]
From the discussion in Section \ref{s:outline} \eqref{e:form}, we have that each summand in $K_t^{(r)}(x_{\al_1}, x_{\al_2},\cdots, x_{\al_r})$ is of the form
\begin{align}\label{e:form3}
\frac{1}{m^{r/2-1}}\prod_{j=1}^s \frac{\langle \sfv_{2j-1}(t),\sfv_{2j}(t)\rangle}{m}, \quad 1\leq s\leq r,\quad \sfv_{j}\in \fD_0\cup \fD_1\cup \cdots \cup \fD_{r-2}.
\end{align}
If $r\leq p+1$, Proposition \ref{p:ftbound} provides an upper bound on the $L_\infty$ norm of those vectors $\sfv_j(t)$. So we can bound these inner products $\langle \sfv_{2j-1}(t),\sfv_{2j}(t) \rangle$ using Proposition \ref{p:ftbound}. If $r\leq p+1$, then for $0\leq t\leq m^{\frac{p}{2(p+1)}}/(\ln m)^{\fC'}$, it holds that 
\begin{align*}
\|\sfv_j(t)\|_\infty\lesssim (\ln m)^\fC.
\end{align*}
As a consequence, with high probability with respect to the random initialization, 
\begin{align*}
\frac{1}{m^{r/2-1}}\prod_{j=1}^s \frac{\langle \sfv_{2j-1}(t),\sfv_{2j}(t)\rangle}{m}
\lesssim 
\frac{1}{m^{r/2-1}}\prod_{j=1}^s \frac{(\ln m)^{2\fC}m}{m}
=\frac{(\ln m)^{2s\fC}}{m^{r/2-1}}\leq \frac{(\ln m)^{2r\fC}}{m^{r/2-1}}.
\end{align*}
Since $K_t^{(r)}(x_{\al_1}, x_{\al_2},\cdots, x_{\al_r})$ is a linear combination of terms in the form \eqref{e:form3}, the claim \eqref{e:tprior2} follows.
\end{proof}

\section{Proof of Corollary \ref{c:change} and \ref{c:zeroloss}, and Theorem \ref{t:main2}}\label{s:cproof}

\begin{proof}[Proof of Corollary \ref{c:change}]
We first derive an upper bound of the kernel $K_t^{(3)}(\cdot, \cdot, \cdot)$, using its derivative
\begin{align}\label{e:derK3}
\del_tK_t^{(3)}(x_{\al_1},x_{\al_2}, x_{\al_3})=-\frac{1}{n} \sum_{\beta=1}^n K^{(4)}_t(x_{\al_1}, x_{\al_2}, x_{\alpha_3}, x_\beta)(f_\beta(t)-y_\beta).
\end{align}
Thanks to \eqref{e:tprior2}, for $0\leq t\leq m^{\frac{p}{2(p+1)}}/(\ln m)^{\fC'}$, it holds that
\begin{align}\label{e:K4bound}
\|K^{(4)}_t\|_{\infty}\lesssim \frac{(\ln m)^\fC}{m}.
\end{align}
\eqref{e:K4bound} combining with \eqref{e:fL2} implies an upper bound of the righthand side of \eqref{e:derK3},
\begin{align}\label{e:derK3b}
\left|\del_tK_t^{(3)}(x_{\al_1},x_{\al_2}, x_{\al_3})\right|\leq \max_{1\leq \beta \leq n}|K^{(4)}_t(x_{\al_1}, x_{\al_2}, x_{\alpha_3}, x_\beta)|\frac{1}{n} \sum_{\beta=1}^n |f_\beta(t)-y_\beta|
\lesssim \frac{(\ln m)^{\fC}}{m}.
\end{align}
\eqref{e:tprior1} gives an upper bound of $K_0^{(3)}(x_{\al_1},x_{\al_2}, x_{\al_3})\lesssim (\ln m)^{\fC}/m$, and \eqref{e:derK3b} gives an upper bound of the derivative of $K_t^{(3)}(x_{\al_1},x_{\al_2}, x_{\al_3})$. They together implies that with high probability
\begin{align}\label{e:Kt3bound}
|K_t^{(3)}(x_{\al_1},x_{\al_2}, x_{\al_3})|\lesssim \frac{(1+t)(\ln m)^{\fC}}{m},
\end{align}
for any $1\leq \al_1, \al_2,\al_3\leq n$.
We recall that
\begin{align}\label{e:derK2}
\del_tK_t^{(2)}(x_{\al_1},x_{\al_2})=-\frac{1}{n} \sum_{\beta=1}^n K^{(3)}_t(x_{\al_1}, x_{\al_2}, x_\beta)(f_\beta(t)-y_\beta).
\end{align}
Similarly as in \eqref{e:derK3b}, we can use \eqref{e:Kt3bound} to upper bound the righthand side of \eqref{e:derK2},
\begin{align*}
\left|\del_tK_t^{(2)}(x_{\al_1},x_{\al_2})\right|\leq \max_{1\leq \beta\leq n}|K^{(3)}_t(x_{\al_1}, x_{\al_2}, x_\beta)|\frac{1}{n} \sum_{\beta=1}^n |f_\beta(t)-y_\beta|\lesssim \frac{(1+t)(\ln m)^{\fC}}{m}.
\end{align*}
This finishes the proof of Corollary \eqref{c:change}.
\end{proof}

\begin{proof}[Proof of Corollary \ref{c:zeroloss}]
Corollary \ref{c:change} gives the change rate for each entry of the NTK up to time $t\leq m^{ \frac{p}{2(p+1)}}/(\ln m)^{\fC'}$,
\begin{align}\label{e:Linfbound}
\|\del_t K_t^{(2)}\|_\infty\lesssim \frac{(1+t)(\ln m)^\fC}{m}.
\end{align}
By integrating both sides of \eqref{e:Linfbound} from $0$ to $t$, we get an $L_\infty$ bound of the change of the NTK,
\begin{align}\label{e:Linfbound2}
\|K_t^{(2)}-K_0^{(2)}\|_{\infty}
\lesssim \frac{t(1+t)(\ln m)^\fC}{m}.
\end{align}
The $L_\infty$ bound in \eqref{e:Linfbound2} can be used to derive a norm bound of the change of the NTK,
\begin{align*}
\|K_t^{(2)}-K_0^{(2)}\|_{2\rightarrow 2}
\leq \|K_t^{(2)}-K_0^{(2)}\|_{\rm F}
\leq n \|K_t^{(2)}-K_0^{(2)}\|_{\infty}
\lesssim \frac{t(1+t)(\ln m)^\fC n}{m}.
\end{align*}

%
%
%

The change of the smallest eigenvalue of the NTK is upper bounded by the change of its norm. If $0\leq t\leq \fc\sqrt{\la m/n}/(\ln m)^{\fC/2}$, with some $\fc>0$ small enough, the change of the norm $\|K_t^{(2)}-K_0^{(2)}\|_{2\rightarrow 2}\leq \la/2$. Combining with
\eqref{e:eigasup}, we conclude that for $0\leq t\leq \fc\sqrt{\la m/n}/(\ln m)^{\fC/2}$
\begin{align}\label{e:eigminbound}
\la_{\min} \left[K_t^{(2)}(x_\al, x_\beta)\right]_{1\leq \al, \beta\leq n}
\geq \la_{\min} \left[K_0^{(2)}(x_\al, x_\beta)\right]_{1\leq \al, \beta\leq n}
-\|K_t^{(2)}-K_0^{(2)}\|_{2\rightarrow 2}\geq \la/2.
\end{align}

From the defining relation \eqref{e:descent} of the NTK and using \eqref{e:eigminbound}, we have
\begin{align}\begin{split}\label{e:L2upK}
\del_t \sum_{\beta=1}^n\|f_\beta(t)-y_\beta\|^2
&=-\frac{1}{n}\sum_{\al, \beta=1}^nK_t^{(2)}(x_\al, x_\beta)(f_\al(t)-y_\al)(f_\beta(t)-y_\beta)\\
&\leq -\frac{\la}{2n}\sum_{\beta=1}^n\|f_\beta(t)-y_\beta\|^2,
\end{split}\end{align}
for $0\leq t\leq \fc\sqrt{\la m/n}/(\ln m)^{\fC/2}$. Especially, \eqref{e:L2upK} implies an exponential decay of the training error,
\begin{align}\label{e:expdecay}
\sum_{\beta=1}^n\|f_\beta(t)-y_\beta\|^2
\leq e^{-\frac{\la t}{2n}}\sum_{\beta=1}^n\|f_\beta(0)-y_\beta\|^2
\lesssim n e^{-\frac{\la t}{2n}},
\end{align}
for $0\leq t\leq  \fc\sqrt{\la m/n}/(\ln m)^{\fC/2}$. It takes time $t\asymp (2n/\la)\ln(n/\varepsilon)$, for the training error in \eqref{e:expdecay} to reach $\varepsilon$. Therefore if 
\begin{align}\label{e:mbound}
\fc\sqrt{\la m/n}/(\ln m)^{\fC/2}\gtrsim(2n/\la)\ln(n/\varepsilon),
\end{align}
the dynamic \eqref{e:dynamic} finds a global minimum, the training error reaches $\varepsilon$ at time $t\asymp (n/\la)\ln(n/\varepsilon)$. For \eqref{e:mbound} to hold, the neural network needs to be wide
\begin{align*}
m\geq \fC'\left(\frac{n}{\la}\right)^3(\ln m)^{\fC}\ln(n/\varepsilon)^2,
\end{align*} 
with some large constant $\fC'$. This finishes the proof of Corollary \ref{c:zeroloss}.
\end{proof}

\begin{proof}[Proof of Theorem \ref{t:main2}]
We have proven in \eqref{e:fL2} that 
\begin{align}\label{e:ftL2}
\sum_{\beta=1}^n (f_\beta(t)-y_\beta)^2=\OO(n).
\end{align}
We recall the a priori estimate \eqref{e:tprior2} that with high probability with respect to the random initialization, for $0\leq t\leq m^{\frac{p}{2(p+1)}}/(\ln m)^{\fC'}$, it holds that 
\begin{align}\label{e:priorcopy}
\|K_t^{(r)}\|_{\infty}\lesssim \frac{(\ln m)^{\fC}}{m^{r/2-1}}.
\end{align}
We have better estimates if $r$ is odd. In fact, thanks to the equations for the dynamic of the NTK \eqref{e:dynamicr},
\begin{align}\begin{split}\label{e:diffb}
\left|\del_tK_t^{(r)}(x_{\al_1},x_{\al_2},\cdots, x_{\al_r})\right|
&\leq \max_{1\leq \beta \leq n}|K^{(r+1)}_t(x_{\al_1}, x_{\al_2}, \cdots, x_{\alpha_r}, x_\beta)|\frac{1}{n} \sum_{\beta} |f_\beta(t)-y_\beta|\\
&\lesssim \frac{(\ln m)^{\fC}}{m^{(r-1)/2}}\sqrt{\frac{1}{n} \sum_{\beta} |f_\beta(t)-y_\beta|^2}\lesssim \frac{(\ln m)^{\fC}}{m^{(r-1)/2}}.
\end{split}\end{align}
Moreover, thanks to \eqref{e:tprior1}, if $r$ is odd, $\fK^{(r)}=0$ and 
\begin{align}\label{e:initialb}
\|K_0^{(r)}\|_{\infty}\lesssim \frac{(\ln m)^{\fC}}{m^{(r-1)/2}}.
\end{align}
The estimates \eqref{e:diffb} and \eqref{e:initialb} together imply that if $r$ is odd, for $t\leq m^{\frac{p}{2(p+1)}}/(\ln m)^{\fC'}$
\begin{align}\label{e:roddbound}
\|K_t^{(r)}\|_{\infty}\lesssim \frac{(1+t)(\ln m)^{\fC}}{m^{(r-1)/2}},
\end{align}
which is slightly better than the estimate \eqref{e:priorcopy}.
%
%
%
%
%

We denote the vector
\begin{align*}
\Delta f(t)=(f_1(t)-\tilde f_1(t), f_2(t)-\tilde f_2(t),\cdots,f_n(t)-\tilde f_n(t))^\top,
\end{align*}
At $t=0$, we have $\Delta f(t)=0$. We denote time $T$ the first time that $\|\Delta f(t)\|_2\geq \sqrt n$, i.e. $T=\inf_{t\geq 0}\{t: \|\Delta f(t)\|_2\geq \sqrt n\}$.
Then for $t\leq T$, we have that 
\begin{align*}
\sum_{\beta=1}^n ( f_\beta(t)-y_\beta)^2=\OO(n),\quad \sum_{\beta=1}^n (\tilde f_\beta(t)-y_\beta)^2=\OO(n).
\end{align*}
Next we study the difference of the original dynamic and the truncated dynamic for $t\leq T$. We show that $\|\Delta f(t)\|_2$ is much smaller than $\sqrt n$, when $t\leq \min\{\fc \sqrt{\la m/n}/(\ln m)^{\fC/2}, m^{\frac{p}{2(p+1)}}/(\ln m)^{\fC'}, T\}$. As a consequence $T\geq \min\{\fc \sqrt{\la m/n}/(\ln m)^{\fC/2}, m^{\frac{p}{2(p+1)}}/(\ln m)^{\fC'}\}$.

Thanks to \eqref{e:tprior2} and \eqref{e:fL2}, we have that
\begin{align*}\begin{split}
&\phantom{{}={}}\left|\del_t (K_t^{(p)}(x_{\al_1}, x_{\al_2}, \cdots, x_{\al_p})-\tilde K_t^{(p)}(x_{\al_1}, x_{\al_2}, \cdots, x_{\al_p}))\right|\\
&\leq \max_{1\leq \beta\leq n} |K_t^{(p+1)}(x_{\al_1}, x_{\al_2},\cdots, x_{\al_p},x_{\beta})|\frac{1}{n}\sum_{\beta=1}^n |f_\beta(t)-y(t)|
\lesssim \frac{1+t}{m^{p/2}}.
\end{split}\end{align*}
Thus for $t\leq m^{\frac{p}{2(p+1)}}/(\ln m)^{\fC'}$ we have
\begin{align*}
\|K_t^{(p)}-\tilde K_t^{(p)}\|_\infty\lesssim \frac{(1+t)t}{m^{p/2}}.
\end{align*}

By taking difference of \eqref{e:dynamicr} and \eqref{e:truncdynamicr}, we have
\begin{align}\begin{split}\label{e:diff}
&\phantom{{}={}}\del_t \left(K_t^{(r)}(x_{\al_1}, x_{\al_2}, \cdots, x_{\al_r})- \tilde K_t^{(r)}(x_{\al_1}, x_{\al_2}, \cdots, x_{\al_r})\right)\\
&=-\frac{1}{n}\sum_{\beta=1}^n\left(K_t^{(r+1)}(x_{\al_1},x_{\al_2}, \cdots, x_{\al_r},x_\beta)(f_\beta(t)-y_\beta)-\tilde K_t^{(r+1)}(x_{\al_1},x_{\al_2}, \cdots, x_{\al_r},x_\beta)(\tilde f_\beta(t)-y_\beta)\right)\\
&=-\frac{1}{n}\sum_{\beta=1}^n\left(\left(K_t^{(r+1)}(x_{\al_1},x_{\al_2}, \cdots, x_{\al_r},x_\beta)-\tilde K_t^{(r+1)}(x_{\al_1},x_{\al_2}, \cdots, x_{\al_r},x_\beta)\right)(\tilde f_\beta(t)-y_\beta)\right.\\
&\phantom{{}={}}\left.+K_t^{(r+1)}(x_{\al_1},x_{\al_2}, \cdots, x_{\al_r},x_\beta)(f_\beta(t)-\tilde f_\beta(t))\right).
\end{split}\end{align}
We estimate the first term on the righthand side of \eqref{e:diff} as
\begin{align}\begin{split}\label{e:tt1}
&\phantom{{}={}}\left|\frac{1}{n}\sum_{\beta=1}^n\left(K_t^{(r+1)}(x_{\al_1},x_{\al_2}, \cdots, x_{\al_r},x_\beta)-\tilde K_t^{(r+1)}(x_{\al_1},x_{\al_2}, \cdots, x_{\al_r},x_\beta)\right)(\tilde f_\beta(t)-y_\beta)\right|\\
&\leq  \|K_t^{(r+1)}-\tilde K_t^{(r+1)}\|_\infty\frac{1}{n}\sum_{\beta=1}^n
|\tilde f_\beta(t)-y_\beta|\lesssim \|K_t^{(r+1)}-\tilde K_t^{(r+1)}\|_\infty,
\end{split}\end{align}
provided that $t\leq \min\{m^{\frac{p}{2(p+1)}}/(\ln m)^{\fC'}, T\}$.
For the second term on the righthand side of \eqref{e:diff}, for $t\leq m^{\frac{p}{2(p+1)}}/(\ln m)^{\fC'}$, it holds
\begin{align}\begin{split}\label{e:tt2}
&\phantom{{}={}}\left|\frac{1}{n}\sum_{\beta=1}^nK_t^{(r+1)}(x_{\al_1},x_{\al_2}, \cdots, x_{\al_r},x_\beta)(f_\beta(t)-\tilde f_\beta(t))\right|\\
&\leq \|K_t^{(r+1)}\|_{\infty}\frac{1}{n}\sum_{\beta=1}^n |f_\beta(t)-\tilde f_{\beta}(t)|
\leq \|K_t^{(r+1)}\|_{\infty}\sqrt{\frac{1}{n}\sum_{\beta=1}^n |f_\beta(t)-\tilde f_{\beta}(t)|^2 }\\
&= \|K_t^{(r+1)}\|_{\infty}\frac{\|\Delta f(t)\|_2}{\sqrt n}\lesssim 
\frac{(\ln m)^\fC}{m^{(r-1)/2}}\left(\frac{1+t}{\sqrt m}\right)^{\bm1_{2\mid r}}\frac{\|\Delta f(t)\|_2}{\sqrt n},
\end{split}\end{align}
where we used \eqref{e:priorcopy} and \eqref{e:roddbound}. The estimates \eqref{e:tt1} and \eqref{e:tt2} together imply
\begin{align*}\begin{split}
&\phantom{{}={}}\left|\del_t (K_t^{(r)}(x_{\al_1}, x_{\al_2}, \cdots, x_{\al_r})- \tilde K_t^{(r)}  (x_{\al_1}, x_{\al_2}, \cdots, x_{\al_r}))\right|\\
&\lesssim  \|K_t^{(r+1)}-\tilde K_t^{(r+1)}\|_\infty+\frac{(\ln m)^\fC}{m^{(r-1)/2}}\left(\frac{1+t}{\sqrt m}\right)^{\bm1_{2\mid r}}\frac{\|\Delta f(t)\|_2}{\sqrt n}.
\end{split}\end{align*}
We integrate both sides, and get
\begin{align*}
\|K_t^{(r)}- \tilde K_t^{(r)} \|_\infty\lesssim \int_{0}^t\|K_s^{(r+1)}-\tilde K_s^{(r+1)}\|_\infty\rd s+\frac{(\ln m)^\fC}{m^{(r-1)/2}}\left(\frac{1+t}{\sqrt m}\right)^{\bm1_{2\mid r}}\frac{\int_0^t \|\Delta f(s)\|_2\rd s}{\sqrt n},
\end{align*}
for $t\leq \min\{m^{\frac{p}{2(p+1)}}/(\ln m)^{\fC'}, T\}$. 
We notice that in our setting $t$ is much smaller than $\sqrt m$. 
Using  $ \tilde K_t^{(p)}=0$ and \eqref{e:priorcopy} for $r=p-1$ and recursively with  $r= (p-2), \cdots, 2,$  we have 
\begin{align*}
\|K_t^{(r)}-\tilde K_t^{(r)} \|_\infty\lesssim \frac{(1+t)t^{p+1-r}}{m^{p/2}}+\frac{(\ln m)^\fC}{m^{(r-1)/2}\sqrt n}\left(\frac{1+t}{\sqrt m}\right)^{\bm1_{2\mid r}}\int_0^t \|\Delta f(s)\|_2\rd s.
\end{align*}
And especially,
\begin{align}\label{e:L2K}
\|K_t^{(2)}-  \tilde K_t^{(2)}  \|_\infty\lesssim \frac{(1+t)t^{p-1}}{m^{p/2}}+\frac{(1+t)(\ln m)^{\fC}}{m\sqrt {n}}\int_0^t \|\Delta f(s)\|_2\rd s.
\end{align}

By taking difference of \eqref{e:dynamic} and \eqref{e:truncdynamicr} we have
\begin{align}\label{e:derDelta}
\del_t \Delta f_\al (t)=\frac{1}{n}\sum_{\beta=1}^n(\tilde K^{(2)}_t(x_\al,x_\beta)-K^{(2)}_t(x_\al, x_\beta))(\tilde f_\beta(t)-y_\beta)-\frac{1}{n}\sum_{\beta=1}^n K^{(2)}_t(x_\al, x_\beta)\Delta f_\beta(t).
\end{align}
We multiply the vector $\Delta f(t)$ on both sides of \eqref{e:derDelta}
\begin{align}\label{e:inner}
\del_t \|\Delta f(t)\|_2^2
&=\frac{1}{n}\langle \Delta f(t), \sum_{\beta=1}^n(\tilde K^{(2)}_t(x_\al,x_\beta)-K^{(2)}_t(x_\al, x_\beta))(\tilde f_\beta(t)-y_\beta)\rangle-\frac{1}{n}\langle \Delta f(t), K^{(2)}_t\Delta f(t)\rangle.
\end{align}
Here we have abused the notation so that 
$\sum_{\beta=1}^n(\tilde K^{(2)}_t(x_\al,x_\beta)-K^{(2)}_t(x_\al, x_\beta))(\tilde f_\beta(t)-y_\beta)$ 
in the above expression is understood  as a vector with the $\alpha$-component being $\sum_{\beta=1}^n(\tilde K^{(2)}_t(x_\al,x_\beta)-K^{(2)}_t(x_\al, x_\beta))(\tilde f_\beta(t)-y_\beta)$. 
For the first term on the righthand side of \eqref{e:inner}, we estimate it using \eqref{e:L2K},
\begin{align}\begin{split}\label{e:ttt1}
&\phantom{{}={}}\frac{1}{n}\langle \Delta f(t), \sum_{\beta=1}^n(\tilde K^{(2)}_t(x_\al,x_\beta)-K^{(2)}_t(x_\al, x_\beta))(\tilde f_\beta(t)-y_\beta)\rangle\\
&\leq \sum_{\al=1}^n |\Delta f_\al(t)|\|\tilde K_t^{(2)}-K^{(2)}_t\|_\infty\frac{1}{n}\sum_{\beta=1}^n|\tilde f_\beta(t)-y_\beta|\\
&\lesssim \left( \frac{(1+t)t^{p-1}}{m^{p/2}}+\frac{(1+t) (\ln m)^{\fC}}{m\sqrt {n}}\int_0^t \|\Delta f(s)\|_2\rd s\right)\sum_{\al=1}^n |\Delta f_\al(t)|\\
&\lesssim \left(\frac{(1+t)t^{p-1}}{m^{p/2}}+\frac{(1+t) (\ln m)^{\fC}}{m\sqrt {n}}\int_0^t \|\Delta f(s)\|_2\rd s \right)\sqrt n\|\Delta f(t)\|_2.
\end{split}\end{align}

For the second term on the righthand side of \eqref{e:inner}, we use the fact that $[K_t^{(2)}(x_\al,x_\beta)]_{1\leq \al, \beta\leq n}$ is positive definite. In fact, in \eqref{e:eigminbound}, we have proven that for $0\leq t\leq \fc \sqrt{\la m /n }/(\ln m)^{\fC/2}$, 
\begin{align*}
\la_{\min} \left[K_t^{(2)}(x_\al, x_\beta)\right]_{1\leq \al, \beta\leq n}
\geq  \la/2.
\end{align*}
Therefore, 
\begin{align}\label{e:ttt2}
-\frac{1}{n}\langle \Delta f(t), K^{(2)}_t\Delta f(t)\rangle
\leq -\frac{\la}{2n}\|\Delta f(t)\|_2^2.
\end{align}
By plugging \eqref{e:ttt1} and \eqref{e:ttt2} into \eqref{e:inner}, and divide both sides by $2\|\Delta f(t)\|_2$, we get
\begin{align}\label{e:derDft}
\del_t \|\Delta f(t)\|_2
\lesssim \sqrt n\left(\frac{(1+t)t^{p-1}}{m^{p/2}}+\frac{(1+t)(\ln m)^{\fC}}{m\sqrt {n}}\int_0^t \|\Delta f(s)\|_2\rd s \right)-\frac{\la }{2n}\|\Delta f(t)\|,
\end{align}
for $t\leq \min\{\fc \sqrt{\la m/n}/(\ln m)^{\fC/2}, m^{\frac{p}{2(p+1)}}/(\ln m)^{\fC'}, T\}$.
To analyze \eqref{e:derDft}, we introduce a new quantity,
\begin{align*}
\Delta(t)=\max_{0\leq s\leq t}\|\Delta f(t)\|_2.
\end{align*}
Then \eqref{e:derDft} implies
\begin{align}\begin{split}\label{e:boundDelta}
\del_t\Delta(t)
&\lesssim \max\left\{0, \sqrt n\left(\frac{(1+t)t^{p-1}}{m^{p/2}}+\frac{(1+t)(\ln m)^{\fC}}{m\sqrt {n}}\int_0^t \Delta(s)\rd s \right)-\frac{\la }{2n}\Delta(t)\right\}\\
&\lesssim\max\left\{0, \left(\frac{(1+t)t^{p-1}\sqrt n}{m^{p/2}}+\frac{(1+t)t(\ln m)^{\fC}}{m} \Delta(t) \right)-\frac{\la }{2n}\Delta(t)\right\}\\
&\lesssim\max\left\{0, \frac{(1+t)t^{p-1}\sqrt n}{m^{p/2}}+\left(\frac{(1+t)t(\ln m)^{\fC}}{m}-\frac{\la}{2n}\right) \Delta(t) \right\},
\end{split}\end{align}
where we used that $\Delta(t)$ is monotonic increasing. We can further simplify the righthand side of \eqref{e:boundDelta}, for $t\leq \min\{\fc \sqrt{\la m/n}/(\ln m)^{\fC/2}, m^{\frac{p}{2(p+1)}}/(\ln m)^{\fC'}, T\}$, where $\fc>0$ is small enough, 
\begin{align}\label{e:boundDelta2}
\del_t\Delta(t)
&\lesssim\max\left\{0, \frac{(1+t)t^{p-1}\sqrt n}{m^{p/2}}-\frac{\la}{4n}\Delta(t) \right\}.
\end{align}
We recall that $\Delta(0)=0$. We can  solve \eqref{e:boundDelta2},
\begin{align*}
\Delta(t)\lesssim e^{-\la t/4n}\int_0^s\frac{(1+s)s^{p-1}\sqrt n}{m^{p/2}} e^{\la s/4n}\rd s\lesssim \frac{(1+t)t^{p-1}\sqrt n}{m^{p/2}} \min\{t, n/\la\}.
\end{align*}

It follows that for $t\leq \min\{\fc \sqrt{\la m/n}/(\ln m)^{\fC/2}, m^{\frac{p}{2(p+1)}}/(\ln m)^{\fC'}, T\}$
\begin{align}\label{e:dfbound}
\|\Delta f(t)\|_2\lesssim \frac{(1+t)t^{p-1}\sqrt n}{m^{p/2}} \min\{t, n/\la\},\end{align}
and 
\begin{align}\label{e:diffKbound}
\|K_t^{(2)}- \tilde K_t^{(2)}  \|_\infty\lesssim 
\frac{(1+t)t^{p-1}}{m^{p/2}}\left(1+\frac{(1+t)t(\ln m)^{\fC}}{m}\min\left\{t, \frac{n}{\la}\right\}\right).
\end{align}
We notice that for $t\leq \min\{\fc \sqrt{\la m/n}/(\ln m)^{\fC/2}, m^{\frac{p}{2(p+1)}}/(\ln m)^{\fC'}, T\}$, and $p\geq 3$, the righthand side of \eqref{e:dfbound} is much smaller than $\sqrt n$. From the definition of $T$, it is necessary that  $T\geq \min\{\fc \sqrt{\la m/n}/(\ln m)^{\fC/2}, m^{\frac{p}{2(p+1)}}/(\ln m)^{\fC'}\}$. Thus we can conclude that \eqref{e:dfbound} and \eqref{e:diffKbound} hold for any $t\leq \min\{\fc \sqrt{\la m/n}/(\ln m)^{\fC}, m^{\frac{p}{2(p+1)}}/(\ln m)^{\fC'}\}$. This finishes the proof of Theorem \ref{t:main2}

\end{proof}

\end{document}